\def\epsdeltastuff{0}
\def\comments{1}

\documentclass[letterpaper,11pt]{article}
\usepackage[utf8]{inputenc}
\usepackage{comment}
\usepackage{preamble-jon}
\allowdisplaybreaks

\newcommand{\nope}[1]{}

\renewcommand{\epsilon}{\varepsilon}
\newcommand{\ball}[2]{\mathit{B}_{#2}\left( #1 \right)}

\newcommand{\RR}{\mathbb{R}}

\newcommand{\tr}{\mathrm{tr}}
\newcommand{\Sigmahat}{\widehat{\Sigma}}

\newcommand{\id}{\mathbb{I}}

\newcommand{\opt}{\mathrm{opt}}

\newcommand{\GUE}{\mathrm{GUE}}
\newcommand{\DPEE}{\mathrm{EigenvalueEstimator}}
\newcommand{\DPCP}{\mathrm{CoarsePreconditioner}}
\newcommand{\DPFP}{\mathrm{FinePreconditioner}}
\newcommand{\DPP}{\mathrm{Preconditioner}}
\newcommand{\DPGCE}{\mathrm{GaussianCovarianceEstimator}}
\newcommand{\PME}{\mathrm{PME}}

\newcommand{\subspace}{\mathrm{SubspaceRecovery}}
\newcommand{\NE}{\mathrm{NaiveEstimator}}

\newcommand{\goodcenter}{\mathrm{GoodCenter}}

\usepackage[style=alphabetic,backend=bibtex,maxalphanames=10,maxbibnames=20,maxcitenames=10,giveninits=false,doi=false,url=true,backref=true]{biblatex}
\newcommand*{\citet}[1]{\AtNextCite{\AtEachCitekey{\defcounter{maxnames}{2}}}\textcite{#1}}
\newcommand*{\citetall}[1]{\AtNextCite{\AtEachCitekey{\defcounter{maxnames}{999}}}\textcite{#1}}
\newcommand*{\citep}[1]{\citep{#1}}

\usepackage{hyperref}

\title{A Private and Computationally-Efficient Estimator for Unbounded Gaussians\thanks{Authors are listed in alphabetical order.}}
\iftrue
\author{Gautam Kamath\thanks{\texttt{g@csail.mit.edu}. Cheriton School of Computer Science, University of Waterloo. Supported by an NSERC Discovery Grant, and a University of Waterloo startup grant.}
    \qquad\qquad\and
    Argyris Mouzakis\thanks{\texttt{amouzaki@uwaterloo.ca}. Cheriton School of Computer Science, University of Waterloo.  Supported by an NSERC Discovery Grant and a David R. Cheriton Graduate Scholarship. }
    \qquad\qquad\and
    Vikrant Singhal\thanks{\texttt{vikrant.singhal@uwaterloo.ca}.  Cheriton School of Computer Science, University of Waterloo.  Part of this research was performed at the Khoury College of Computer Sciences, Northeastern University.  Supported by NSF grants CCF-1750640, CNS-1816028, and CNS-1916020, and an NSERC Discovery Grant.}
    \and
    Thomas Steinke\thanks{\texttt{badgauss@thomas-steinke.net}. Google Research, Brain Team.}
    \qquad\qquad\and
    Jonathan Ullman\thanks{\texttt{jullman@ccs.neu.edu}.  Khoury College of Computer Sciences, Northeastern University.   Affiliated with the Institute for Experiential AI and the Cybersecurity \& Privacy Institute.  Supported by NSF grants CCF-1750640, CNS-1816028, and CNS-1916020. }}
\date{}
\fi

\begin{document}
\maketitle
\thispagestyle{empty}

\begin{abstract}
    We give the first polynomial-time, polynomial-sample, differentially private estimator for the mean and covariance of an arbitrary Gaussian distribution $\cN(\mu,\Sigma)$ in $\R^d$.  All previous estimators are either nonconstructive, with unbounded running time, or require the user to specify a priori bounds on the parameters $\mu$ and $\Sigma$.  The primary new technical tool in our algorithm is a new differentially private preconditioner that takes samples from an arbitrary Gaussian $\cN(0,\Sigma)$ and returns a matrix $A$ such that $A \Sigma A^T$ has constant condition number.
\end{abstract}
\newpage

\setcounter{page}{1}
\section{Introduction}
All useful statistical estimators have the side effect of revealing information about their sample, which leads to concerns about the \emph{privacy} of the individuals who contribute their data to the sample.  In this work we study statistical estimation with the constraint of \emph{differential privacy (DP)}~\cite{DworkMNS06}, a rigorous individual privacy criterion well suited to statistical estimation and machine learning.


As in classical statistical estimation, it is impossible to privately estimate even basic statistics like the mean and covariance without some restrictions on the distribution, although the assumptions made in the private setting are typically stronger both qualitatively and quantitatively.  To provide some intuition for the assumptions required for private estimation, consider the simple problem of privately estimating the mean of a distribution $\cD$ over $\R^d$ from a set of $n$ samples $X_1, \dots, X_n \sim \cD$. The standard way to solve this problem is by computing a noisy empirical mean
\[\hat \mu = \frac1n \sum_{i=1}^n X_i + Z,\]
where $Z$ is a suitable random variable---typically Gaussian or Laplacian.  The magnitude of $Z$ must be proportional to the \emph{sensitivity} of $\frac{1}{n} \sum_{i=1}^n X_i$, which measures how much its value can change if a single point $X_i$ is modified arbitrarily.  Without further information about the underlying distribution, the sensitivity is infinite, rendering this na\"ive approach ineffective.

To facilitate using a low-sensitivity mean estimator, we generally make two types of assumptions on the underlying distribution $\cD$:
\begin{enumerate}
  \item The distribution $\cD$ is somehow well-behaved. For example, we assume $\cD$ is a Gaussian distribution $\mathcal{N}(\mu, \Sigma)$, while other works have assumed weaker moment bounds~\cite{BarberD14, BunS19, KamathSU20}.
  \item The analyst has some prior knowledge about the parameters of the distribution $\cD$. The standard assumption in this setting is that the analyst knows parameters $R>0$ and $K>0$ such that $\|\mu\|_2 \leq R$ and $\mathbb{I} \preceq \Sigma \preceq K\mathbb{I}$.\footnote{Here, $A \preceq B$ refers to the PSD order denoting that $x^T A x \leq x^T B x$ for every $x \in \R^d$, and $\mathbb{I}$ denotes the identity matrix.}\footnote{By translating and rescaling the distribution, these assumptions can be relaxed to $\| \mu - c \|_2 \leq R$ for some known vector $c$ and $\mathbb{I} \preceq A \Sigma A^T \preceq K \mathbb{I}$ for some known matrix $A$.}
\end{enumerate}
These assumptions ensure that we can identify a finite subset of the domain that contains all the samples with high probability, which we can use to find a proxy for the empirical mean with finite sensitivity.

The first style of assumption is common and generally necessary to provide non-trivial guarantees even in the non-private setting. The second style of assumption however is particular to the private setting, and forces the analyst to input some prior knowledge about the location and shape of their data. This may be a minimal burden to place on the user when the domain is familiar, but can be unreasonable for unfamiliar, high-dimensional domains. In that case the analyst may only be able to give extremely loose bounds, corresponding to very large values of $R$ and $K$. This leads to a degradation of the accuracy of the final output.

For these reasons, a key goal in private algorithm design is minimizing the sample complexity's dependence on the prior knowledge in the form of the parameters $R$ and $K$. Na\"ive algorithms limit the empirical estimator's sensitivity by simply clipping the data based on the analyst's prior knowledge, incurring an undesirable linear dependence on $R$ and $K$. More clever approaches iteratively refine the analyst's knowledge of the shape and location of the distribution.  That is, we start by finding a weak estimate of the parameters $\mu$ and $\Sigma$, which allows us to rescale the data and thereby reduce the parameters $R$ and $K$ for the next steps. This approach results in improved sample complexity compared to the na\"ive strategy outlined above:  For the univariate case, it can be used to eliminate the dependence on $R$ and $K$ entirely~\cite{KarwaV18}.  For the multivariate case, existing approaches yield a polylogarithmic dependence on $R$ and $K$~\cite{KamathLSU19}---an exponential improvement---but do not eliminate the need for a priori bounds.

Despite exponential improvements, it is natural to wonder whether a dependence on $R$ and $K$ is necessary at all.
For more restrictive special cases of differential privacy, such as pure or concentrated differential privacy,\footnote{Though we later define the various relevant notions of DP, we remind the reader that pure $(\varepsilon, 0)$-DP is stronger than concentrated DP, which in turn is stronger than approximate $(\varepsilon, \delta)$-DP.} packing lower bounds imply that a polylogarithmic dependence is the best possible~\cite{BunS16, BunKSW19}. However, these lower bounds do not apply to the most general notion of approximate differential privacy, and in this model we can often eliminate the need for any a priori bounds on the distribution, which is clearly an appealing feature of an estimator.

For mean estimation, it is relatively easy to eliminate the need for a priori bounds on the mean (the parameter $R$), but the rich geometric structure of covariance matrices makes it much more challenging to eliminate the need for bounds on the covariance (the parameter $K$), even without requiring computational efficiency.  
Recently, building on a cover-based technique of \cite{BunKSW19}, \cite{AdenAliAK21} show the existence of an estimator that doesn't require any bounds on the covariance matrix, but their argument is non-constructive and does not give an estimator with polynomial, or even finite running time.

\subsection{Results}
Our main result is a polynomial-time algorithm for Gaussian estimation which requires no prior knowledge about the distribution parameters. 

\begin{thm}[Informal] \label{thm:main-est-intro}
    There is a polynomial-time $(\eps,\delta)$-differentially private estimator $M$ with the following guarantee: For every $\mu \in \R^d$ and positive semidefinite $\Sigma \in \R^{d \times d}$, if $X_1,\dots,X_n \sim \cN(\mu,\Sigma)$ and 
    $$
        n \geq \tilde{O}\left(\frac{d^2}{\alpha^2} + \frac{d^2 \cdot \mathrm{polylog}(1/\delta)}{\alpha \eps} + \frac{d^{5/2} \cdot \mathrm{polylog}(1/\delta)}{\eps} \right),
    $$
    then, with high probability, $M(X_1,\cdots,X_n)$ outputs $\hat\mu \in \R^d$ and $\hat\Sigma \in \R^{d \times d}$ such that
    $$
        \| \hat\Sigma - \Sigma \|_{\Sigma} := \| \Sigma^{-1/2} \hat\Sigma \Sigma^{-1/2} - \mathbb{I} \|_{F} \leq \alpha
    $$
    and
    $$
        \| \hat\mu-\mu \|_\Sigma := \| \Sigma^{-1/2}\hat\mu - \Sigma^{-1/2}\mu \|_2 \le \alpha.
    $$
    In particular, this guarantee implies that $\cN(\hat\mu,\hat\Sigma)$ and $\cN(\mu,\Sigma)$ are $O(\alpha)$-close in total variation distance.
\end{thm}

The main advantage of our result compared to prior work is that our estimator both runs in polynomial time and requires no prior bounds on $\Sigma$, whereas all estimators from prior work lack at least one of these properties.
The best known sample complexity is the result of \cite{AdenAliAK21}, which is $n=O(d^2/\alpha^2 + d^2/\alpha\eps + \log (1/\delta)/\eps)$.
Our estimator has a slightly worse dependence on the dimension $d$, but our running time is polynomial instead of unbounded, and it remains open to find a polynomial-time estimator with information-theoretically optimal sample complexity.
Their bound is conjectured to be tight, but matching lower bounds under $(\eps, \delta)$-DP are only known for the first and third terms.
A lower bound of $\Omega(d^2/\alpha\eps)$ has only been proven under $(\eps, 0)$-DP.
See Section 1.1.1 of~\cite{AdenAliAK21} for more discussion on lower bounds. 
See Table~\ref{tab:cov-estimators} for more information on prior upper bounds.

Several concurrent works have appeared after the preprint of our work, which also achieve similar results. See the discussion of Simultaneous and Subsequent Work in Section~\ref{sec:related}.
 
\begin{table}[]
\centering
\begin{tabular}{@{}lll@{}}
\toprule
Reference                           & Sample Complexity                                  & Computational Complexity \\ \midrule
Non-Private                         & $\frac{d^2}{\alpha^2}$                             & Polynomial   \\[6pt]
Na\"ive Estimator    & $\frac{d^2}{\alpha^2} + \frac{K d^2}{\alpha \eps}$ & Polynomial   \\[4.5pt]
\citet{KamathLSU19}                  & $\frac{d^2}{\alpha^2} + \frac{d^2}{\alpha \eps} + \frac{d^{3/2}\log^{1/2} K}{\eps}$ & Polynomial \\[6pt]
\citet{AdenAliAK21}                  & $\frac{d^2}{\alpha^2} + \frac{d^2}{\alpha \eps}$   & Unbounded    \\[4.5pt]
Theorem \ref{thm:main-est-intro} (this work) & $\frac{d^2}{\alpha^2} + \frac{d^2}{\alpha \eps} + \frac{d^{5/2}}{\eps}$            & Polynomial \\[4.5pt]
\citet{AshtianiL21} (concurrent) & $\frac{d^2}{\alpha^2} + \frac{d^2}{\alpha \eps}$            & Polynomial \\[4.5pt]
\citet{KothariMV21} (concurrent) & $\frac{d^8}{\alpha^4 \varepsilon^8}$            & Polynomial \\ 
\end{tabular}
\caption{Comparing $(\eps,\delta)$-differentially private covariance estimators for $\cN(0,\Sigma)$.  Here, $d$ is the dimension, $K$ is an a priori bound such that $\mathbb{I} \preceq \Sigma \preceq K\mathbb{I}$, and the accuracy guarantee is $\| \hat\Sigma - \Sigma \|_{\Sigma} \leq \alpha$.  The sample-complexity bounds suppress polylogarithmic factors in $d, \frac{1}{\alpha},$ and $\frac{1}{\delta}$. }
\label{tab:cov-estimators}
\end{table}

\subsection{Overview of Techniques}

Our algorithm builds on the \emph{private preconditioning} framework introduced in \cite{KamathLSU19}.  Here our goal is to privately obtain a matrix $A$ such that, after rescaling, $\mathbb{I} \preceq A \Sigma A^T \preceq O(1) \cdot \mathbb{I}$.  The preceding statement implicitly assumes that $\Sigma$ is full rank, which is useful to simplify the discussion, but our methods also handle the more general case of a degenerate covariance matrix $\Sigma$.  Given such a matrix $A$, we can perform the invertible transformation of replacing each sample $X_i$ with $AX_i$ and then apply the na\"ive private estimator to these transformed samples and finally invert the transformation to obtain our estimates $\hat\mu$ and $\hat\Sigma$.  Since $X \sim \cN(\mu, \Sigma)$ implies $AX \sim \cN(A\mu,A\Sigma A^T)$,  we now have a good a priori bound on the covariance $A \Sigma A^T$ and, hence, the na\"ive estimator will have small sample complexity.

The main technical ingredient in our estimator is a new \emph{private preconditioner} that takes samples of the form $X \sim \cN(0,\Sigma)$, for an arbitrary $\Sigma$, and outputs a matrix $A$ so that $A \Sigma A^T$ is well conditioned.\footnote{Without loss of generality, we can restrict our attention to the case where the data is drawn from $\mathcal{N}(\mu,\Sigma)$ with $\mu=0$. If we are given two independent samples $X,X' \sim \mathcal{N}(\mu,\Sigma)$, then $(X-X')/\sqrt{2}$ has the distribution of $\mathcal{N}(0,\Sigma)$.}  
\begin{thm}[Informal] \label{thm:main-precond-intro}
    There is a polynomial-time $(\eps,\delta)$-differentially private algorithm $M$ with the following guarantee: For every positive-semidefinite, rank-$k$ matrix $\Sigma \in \R^{d \times d}$, if $X_1,\dots,X_n \sim \cN(0,\Sigma)$ and
    $$
        n \geq \tilde{O}\left(\frac{d^{5/2} \cdot \mathrm{polylog}(1/\delta)}{\eps} \right),
    $$
    then, with high probability, $M(X_1,\dots,X_n)$ outputs $A \in \R^{d \times d}$ such that $\frac{\lambda_{1}(A\Sigma A^T)}{ \lambda_{k}(A \Sigma A^T)} = O(1)$, where we write $\lambda_1 \geq \lambda_2 \geq \dots \geq \lambda_d$ for the sorted eigenvalues of the matrix.
\end{thm}

To contrast Theorem~\ref{thm:main-precond-intro} with that of \cite{KamathLSU19}, their work gave a polynomial-time algorithm that takes samples from a Gaussian $\cN(0,\Sigma)$ such that $\id \preceq \Sigma \preceq K \id$ and returns a matrix $A$ such that $\id \preceq A \Sigma A^T \preceq \frac{K}{2} \id$.  Thus, iteratively applying their preconditioner $O(\log K)$ times and using composition bounds for differential privacy gives a result similar to Theorem~\ref{thm:main-precond-intro}, but with a $(\log K)^{1/2}$ term in the sample complexity.  In contrast, very informally, our preconditioner is able to find a good estimate of $\Sigma$ one direction at a time, no matter how poorly conditioned $\Sigma$ is, so the number of iterations depends only on the dimension $d$ and not on any assumptions about $\Sigma$ itself.

Since the preconditioner of \cite{KamathLSU19} can already handle the case where the condition number $K$ is small or moderately large, the main technical hurdle that our work must overcome is the case where the condition number is very large, specifically exponential: $\lambda_d(\Sigma)/\lambda_1(\Sigma) \le \exp(-\poly(d))$.
When the eigenvalues of $\Sigma$ are so spread out, there must be a large eigenvalue gap where $\lambda_{k+1}(\Sigma)/\lambda_{k}(\Sigma)$ is very small, at most inverse-polynomial in $d$.  Thus, the key technical ingredient we need is a private algorithm that can output an approximation to the $k$-dimensional subspace of $\Sigma$ containing the directions of large variance.  Given such a subspace, we can partition the space into a $k$-dimensional subspace where the covariance is well conditioned and a lower-dimensional subspace, and then recur on the lower-dimensional subspace.  This \emph{private subspace recovery} problem has been investigated before, originally by \cite{DworkTTZ14}, and, recently \cite{SinghalS21} gave an algorithm for this problem that gives dimension-independent sample complexity under the assumption of a large eigenvalue gap between the top-$k$ subspace and its complement.  In order to apply their algorithm in our setting, we give a different analysis,
and along the way we make other modifications that, for our application, reduce the sample complexity by polynomial factors in the dimension.
\begin{thm}[Informal, extension of \cite{SinghalS21}]\label{thm:main-subspace-intro}
    There is a polynomial-time $(\eps,\delta)$-differentially private algorithm $M$ with the following guarantee: Let $\Sigma \in \R^{d \times d}$ such that $\lambda_{k+1}(\Sigma) / \lambda_{k}(\Sigma) < \gamma^2$ for some $1 \leq k < d$ and $0 < \gamma \leq 1$, and let $\Pi \in \R^{d \times d}$ be the matrix that projects onto the subspace spanned by the top-$k$ eigenvectors of $\Sigma$. If $0 < \psi \leq 1$ and $X_1,\dots,X_n \sim \cN(0,\Sigma)$ and
    $$
        n \geq \tilde{O}\left(\frac{d^{3/2} k^{1/2} \cdot \mathrm{polylog}(1/\delta)}{\psi^2 \eps} \right),
    $$
    then with high probability, $M(X_1,\dots,X_n)$ outputs a projection matrix $\hat\Pi \in \R^{d \times d}$ such that $\|\hat\Pi - \Pi \|_2 \leq \psi \gamma$.
    
\end{thm}

The subspace recovery algorithm of \cite{SinghalS21} is tailored to allow a dimension-independent sample
complexity, which is something that our modifications no longer achieve.
However, in our setting, a direct application of their
algorithm would be inefficient in terms of the
sample complexity. Here, we are free to pick
$\poly(d)$ samples, which gives us the option to use more accurate methods
in the subspace recovery algorithm -- we trade $\poly(d)$
sample complexity for improved accuracy. In particular, we incorporate the ball-finding algorithm
of \cite{NissimSV16}. Roughly speaking, if the
eigengap is $\gamma^2$, then to get an error proportional
to $\gamma$, \cite{SinghalS21} would require $O(d^2k^2)$
samples, while our modifications reduce this cost to
$O(d^{3/2}k^{1/2})$.


\subsection{Related Work}\label{sec:related}

Differentially private statistical inference has been an active area of research for over a decade (e.g.~\cite{DworkL09, VuS09, WassermanZ10, Smith11}), and the literature is too broad to fully summarize here.  Our work fits into two more recent trends that we survey below---designing private estimators without the need for strong prior bounds and pinning down the minimax sample complexity for differentially private estimation.

\mypar{Private Estimation without Prior Knowledge.}
The influential work \cite{KarwaV18} focused attention on minimizing the need for prior knowledge as a key issue for obtaining practical private estimators, providing both algorithms and lower bounds for univariate Gaussian mean and variance estimation.  In particular, they designed pure DP estimators with a logarithmic dependence on the bounding parameters using a general recipe based on private histograms, and estimators with approximate DP with no dependence on these parameters.  Subsequent works gave other pure DP or concentrated DP algorithms for the univariate case with a similar logarithmic dependence, based on techniques such as the exponential mechanism~\cite{DuFMBG20}, iteratively shrinking confidence intervals~\cite{BiswasDKU20}, the trimmed mean \cite{BunS19}, and quantile estimation~\cite{HuangLY21}.  Other techniques have been employed to deal with the bounding parameters for univariate median estimation~\cite{AvellaMedinaB19, TzamosVZ20}, including propose-test-release~\cite{DworkL09} and efficient Lipschitz extensions~\cite{CummingsD20,TzamosVZ20}.

All the above techniques for univariate mean estimation extend to multivariate mean estimation with known covariance, simply by applying a univariate estimator to each coordinate, however extending to multivariate covariance estimation is significantly more challenging. \cite{KamathLSU19} gave the first algorithm for this setting which satisfies concentrated DP or approximate DP, and incurs only a logarithmic dependence on the bounding parameters, which was subsequently refined into a more practical variant~\cite{BiswasDKU20}.
\cite{BunKSW19} provides a cover-based approach which leads to pure DP algorithms for more general settings with logarithmic dependence on the bounding parameters, but the estimators have exponential running time or worse.
They further provide an approach for proving approximate DP sample complexity bounds which require no bounding parameters, contingent on the construction of a locally-sparse cover.
As they describe it, their method has an infinite running time, and they are also only able to construct such a cover for multivariate Gaussians with known covariance, as the rich geometric structure makes the unknown covariance case hard to reason about.
\cite{AdenAliAK21} extends this approach to require only a collection of sparse local covers, allowing them to prove a bound on the sample complexity of covariance estimation with no bounding parameters.
Again, their approach does not provide even a finite-time algorithm, and our result is the first polynomial-time algorithm for covariance estimation with no dependence on the bounding parameters.
Recent work~\cite{BrownGSUZ21} provides an approach for Gaussian mean estimation with unknown covariance, which bypasses the problem of covariance estimation to obtain better sample complexity.  Specifically, they provide a computationally-inefficient approximate DP algorithm which requires no parameter knowledge.
Since our goal is to estimate the covariance, their results are inapplicable to our setting.

\mypar{Minimax Sample Complexity.}
Our work also falls into a broader line of work on minimax sample complexities for differentially private statistical estimation.  See~\cite{KamathU20} for a partial survey of this line of work.
The first minimax sample complexity bounds to show an asymptotic separation between private and non-private estimation for private mean estimation were proven in~\cite{BunUV14}, and subsequently sharpened and generalized in several respects~\cite{DworkSSUV15,BunSU17,SteinkeU17a,SteinkeU17b,KamathLSU19}.  More recently, \cite{CaiWZ19} extended these bounds to sparse estimation and regression problems.  \cite{AcharyaSZ21} provides an alternative, user-friendly approach to proving sample complexity bounds, which is directly analogous to the classical approaches for proving minimax lower bounds in statistics.  These approaches are less powerful in general, but yields tight bounds for certain statistical estimation tasks.

There are a wide variety of results pinning down the minimax sample complexity for estimation under a variety of distributional assumptions, including settings with heavy-tailed data~\cite{BarberD14,BunS19, KamathSU20, WangXDX20, KamathLZ21,HopkinsKM22}, mixtures of Gaussians~\cite{KamathSSU19, AdenAliAL21}, graphical models~\cite{ZhangKKW20}, and discrete distributions~\cite{DiakonikolasHS15}.
Additionally, \cite{LiuKKO21, LiuKO21,HopkinsKM22} give algorithms for mean estimation which are simultaneously private and robust.
Some recent works~\cite{LiuSYKR20, LevySAKKMS21} focus on estimation in a setting where a single person may contribute multiple samples (but privacy must still be provided with respect to all of a person's records).  One work~\cite{AventDK20} studies mean estimation in a hybrid model where some users require the more stringent local DP property, while other are content with central DP.

\mypar{Simultaneous and Subsequent Work.} The initial online posting of this work was accompanied by a flurry of simultaneous and independent papers featuring results on private covariance estimation.
Most directly comparable with our work are the simultaneous and independent results of Ashtiani and Liaw~\cite{AshtianiL21}, and Kothari, Manurangsi, and Velingker~\cite{KothariMV21}, which obtain computationally-efficient algorithms for private estimation of unbounded Gaussians.
Both are also robust to adversarial corruptions.
The techniques of all three works differ from each other, and thus offer multiple perspectives on how to address this problem. 
While our work employs ideas from private subspace recovery, \cite{AshtianiL21} uses a framework based on privately checking whether the results of several non-private estimates resemble each other (a la Propose-Test-Release~\cite{DworkL09}), and \cite{KothariMV21} privately adapts convex relaxations which have recently seen use in robust statistics.
Focusing on the dependence on the dimension $d$, our algorithm has sample complexity $\tilde O(d^{2.5})$, while~\cite{AshtianiL21} is $\tilde O(d^2)$ and~\cite{KothariMV21} is $\tilde O(d^8)$.

Also simultaneous to all these works, Tsfadia, Cohen, Kaplan, Mansour, and Stemmer~\cite{TsfadiaCKMS21} provided a framework similar to that of Ashtiani and Liaw's~\cite{AshtianiL21}, and applied it to the problem of mean estimation.
In a subsequent update, \cite{TsfadiaCKMS21} showed that their approach too can give an efficient (non-robust) private algorithm for estimation of unbounded Gaussian covariances.

Finally, simultaneous and independent to our work, Liu, Kong, and Oh~\cite{LiuKO21} give a framework for designing private estimators via connections with robustness. 
For the specific case of Gaussian covariance estimation, they give a computationally inefficient algorithm with similar guarantees as the work of Aden-Ali, Ashtiani, and Kamath~\cite{AdenAliAK21}.

\subsection{Organization of the Paper}

We start by giving standard background on differential privacy and concentration-of-measure in Section~\ref{sec:prelims}.
After that, we present the algorithm for private eigenvalue estimation in Section~\ref{sec:eigenvalues}.
It is followed by our extended subspace-recovery algorithm in Section~\ref{sec:subspace}.
Next is our main procedure, which performs private preconditioning, in Section~\ref{sec:preconditioner}.
Finally in Section~\ref{sec:final}, we put all our results together to present an algorithm to learn Gaussian covariance.
We describe the remaining subroutine for our main algorithms, the na\"ive estimator, in Appendix~\ref{sec:naive}.

\section{Preliminaries}\label{sec:prelims}

\subsection{Differential Privacy Preliminaries}

A \emph{dataset} $X = (X_1,\dots,X_n) \in \cX^n$ is a collection
of elements from some \emph{universe}.  We say that two datasets
$X,X' \in \cX^n$ are \emph{neighboring} if they differ on at most
a single entry, and denote this by $X \sim X'$.
\begin{defn}[Differential Privacy (DP)~\cite{DworkMNS06}] A randomized algorithm $M: \cX^n \rightarrow \cY$ satisfies \emph{$(\eps,\delta)$-differential privacy ($(\eps,\delta)$-DP)} if for every pair of neighboring datasets $X, X' \in \cX^n$,
$$
\forall Y \subseteq \cY~~~\pr{}{M(X) \in Y} \leq e^{\eps} \pr{}{M(X') \in Y} + \delta.
$$
\end{defn}

\noindent This definition is closed under post-processing
\begin{lem}[Post-Processing~\cite{DworkMNS06}]\label{lem:postprocessing}
    If $M : \cX^n \to \cY$ is $(\eps,\delta)$-DP and
    $P : \cY \to \cZ$ is any randomized function, then
    the algorithm $P \circ M$ is $(\eps,\delta)$-DP.
\end{lem}

A crucial property of all the variants of differential privacy is that they can be composed adaptively.  By adaptive composition, we mean a sequence of algorithms $M_1(X),\dots,M_T(X)$ where the algorithm $M_t(X)$ may also depend on the outcomes of the algorithms $M_1(X),\dots,M_{t-1}(X)$.
\begin{lem}[Composition of DP~\cite{DworkMNS06, DworkRV10, BunS16}] \label{lem:dpcomp}
    If $M$ is an adaptive composition of differentially
    private algorithms $M_1,\dots,M_T$, then the following
    all hold:
    \begin{enumerate}
        \item If $M_1,\dots,M_T$ are
            $(\eps_1,\delta_1),\dots,(\eps_T,\delta_T)$-DP
            then $M$ is $(\eps,\delta)$-DP for
            $\eps = \sum_t \eps_t$ and $\delta = \sum_t \delta_t$
        \item If $M_1,\dots,M_T$ are $(\eps_0,\delta_1),\dots,(\eps_0,\delta_T)$-DP
            for some $\eps_0 \leq 1$, then for every $\delta_0 > 0$,
            $M$ is $(\eps, \delta)$-DP for
            $$\eps = \eps_0 \cdot \sqrt{6 T \log(1/\delta_0)}~~~~\textrm{and}~~~~\delta = \delta_0 + \sum_t \delta_t$$
    \end{enumerate}
\end{lem}

Note that the first property says that $(\eps,\delta)$-DP composes
linearly---the parameters simply add up.  The second property says
that $(\eps,\delta)$-DP actually composes sublinearly---the parameter
$\eps$ grows roughly with the square root of the number of steps in
the composition, provided we allow a small increase in $\delta$.

\subsubsection{Useful Differentially Private Mechanisms}
Our algorithms will extensively use the well known and standard Gaussian mechanism to ensure differential privacy.

\begin{defn}[$\ell_2$-Sensitivity]
Let $f : \cX^n \to \R^d$ be a function, its \emph{$\ell_2$-sensitivity} is
$$
\Delta_{f} = \max_{X \sim X' \in \cX^n} \| f(X) - f(X') \|_{2}
$$
\end{defn}

\begin{lem}[Gaussian Mechanism] \label{lem:gaussiandp}
    Let $f : \cX^n \to \R^d$ be a function
    with $\ell_2$-sensitivity $\Delta_{f}$.
    Then the Gaussian mechanism
    $$M(X) = f(X) + \cN\left(0,\frac{2 \Delta_{f}^2
        \ln(2/\delta)}{\eps^2} \cdot \id_{d \times d}\right)$$
    satisfies $(\eps,\delta)$-DP.
\end{lem}

Next, we describe a tool to privately estimate histograms.
\begin{lem}[Stability-based Histograms~\cite{KorolovaKMN09,BunNS16,Vadhan17}]\label{lem:priv-hist}
    Let $(X_1,\dots,X_n)$ be samples in some data universe
    $U$, and let $\Omega = \{h_u\}_{u \subset U}$
    be a collection of disjoint histogram buckets over $U$.
    Then we have an $(\eps,\delta)$-DP
    histogram algorithm with the following guarantees:
    \begin{itemize}
        \item With probability at least $1-\beta$, the $\ell_\infty$ error is $O\left(\tfrac{\log(1/\delta\beta)}{\eps}\right)$.
        \item The algorithm runs in time $\poly\left(n,\log\left(\frac{1}{\eps\beta}\right)\right)$.
    \end{itemize}
\end{lem}

Finally, we provide a tool to find an approximately
smallest  ball that contains all the points in the
dataset with high probability.
\begin{thm}[$\goodcenter$ from \cite{NissimSV16}]\label{thm:goodcenter}
    Let $X=(X_1,\dots,X_n) \in \RR^D$ be the dataset such
    that
    $$n \geq O\left(\frac{\sqrt{d}\cdot
        \polylog(D,\frac{1}{\eps},\frac{1}{\delta},\frac{1}{\beta})}{\eps}\right).$$
    Suppose the smallest ball in $\RR^D$ that contains all the points
    in $X$ has radius $R_{\opt}$. Then for all $\eps,\delta,\beta>0$,
    there exists an $(\eps,\delta)$-DP algorithm ($\goodcenter$)
    that takes $X,R_{\opt}$ as input, and outputs a point
    $c \in \RR^D$, such that $\ball{c}{CR_{\opt}\sqrt{\log n}}$ (for a universal
    constant $C$) contains at least $\tfrac{n}{2}$ points from $X$
    with probability at least $1-\beta$.
\end{thm}

\subsection{Distribution Estimation Preliminaries}

In this work, our goal is to estimate some underlying distribution in total variation distance.
We will achieve this by estimating the parameters of the distribution, and we argue that a distribution from the class with said parameters will be accurate in total variation distance.
For a vector $x$, define $\|x\|_\Sigma = \|\Sigma^{-1/2}x\|_2$.
Similarly, for a matrix $X$, define $\|X\|_\Sigma = \|\Sigma^{-1/2}X\Sigma^{-1/2}\|_F$.
With these two norms in place, we have the following lemma, which is a combination of Corollaries 2.13 and 2.14 of~\cite{DiakonikolasKKLMS16}.
\begin{lem}
\label{lem:gaussian-tv}
Let $\alpha \geq 0$ be smaller than some absolute constant.
Suppose that $\|\mu - \hat \mu\|_\Sigma \leq \alpha$, and $\|\Sigma - \hat \Sigma\|_\Sigma \leq \alpha$, where $\mathcal{N}(\mu, \Sigma)$ is a Gaussian distribution in $\mathbb{R}^d$, $\hat \mu \in \mathbb{R}^d$, and $\Sigma \in \mathbb{R}^{d \times d}$ is a PSD matrix.
Then $\dtv(\mathcal{N}(\mu, \Sigma), \mathcal{N}(\hat \mu, \hat \Sigma)) \leq O(\alpha)$.
\end{lem}

\subsubsection{Useful Inequalities}

We will need several facts about Gaussians and Gaussian matrices.  Throughout this section, let $\GUE(\sigma^2)$ denote the distribution over $d \times d$ symmetric matrices $M$ where for all $i \leq j$, we have $M_{ij} \sim \cN (0, \sigma^2)$ i.i.d..
From basic random matrix theory, we have the following guarantee.
\begin{thm}[see e.g. \cite{Tao12} Corollary 2.3.6]\label{thm:GUE}
    For $d$ sufficiently large, there exist absolute constants $C, c > 0$ such that
    \[
        \pr{M \sim \GUE (\sigma^2) }{\| M \|_{2} > A \sigma \sqrt{d}} \leq C \exp (-c A d)
    \]
    for all $A \geq C$.
\end{thm}
\noindent
We also require the following, well known tail bound on quadratic forms on Gaussians.
\begin{thm}[Hanson-Wright Inequality (see e.g.~\cite{LaurentM00})]\label{thm:hanson-wright}
    Let $X \sim \cN (0, \id)$ and let $A$ be a $d \times d$ matrix.
    Then, for all $t > 0$, the following two bounds hold:
    \begin{align}
        &\pr{}{X^\top A X - \tr (A) \geq 2 \| A \|_F \sqrt{t} + 2 \| A \|_2 t} \leq \exp (-t) \label{eq:hw-ub} \\
        &\pr{}{X^\top A X - \tr (A) \leq -2 \| A \|_F \sqrt{t}} \leq \exp (-t)\label{eq:hw-lb}
    \end{align}
\end{thm}
\noindent As a special case of the above inequality, we also have the following.
\begin{fact}[\cite{LaurentM00}]\label{fact:chi-squared}
    Fix $\beta > 0$, and let $X_1, \ldots, X_m \sim \cN (0, \sigma^2)$ be independent.
    Then
    \[
        \pr{}{\left| \frac{1}{m} \sum_{i = 1}^m X_i^2 - \sigma^2 \right| > 4 \sigma^2 \left( \sqrt{\frac{\log(1/\beta)}{m}} + \frac{2 \log(1/\beta)}{m} \right)} \leq \beta
    \]
\end{fact}

Now, we state an inequality bounding the eigenvalues
of sum of two matrices.
\begin{lem}[Weyl's Inequality]\label{lem:weyl}
    Let $M,N,R$ be $d\times d$ Hermitian matrices, such
    that $M=N+R$. Then for each $1 \leq i \leq d$,
    $$\lambda_i(N) + \lambda_d(R) \leq \lambda_i(M) \leq \lambda_i(N) + \lambda_1(R).$$
\end{lem}

\noindent In order to prove accuracy, we will use
the following standard tail bounds for Gaussian random
variables.

\begin{lem}\label{lem:gaussian-error}
    If $Z \sim \cN(0,\sigma^2)$ then for every $t > 0$,
    $\pr{}{|Z| > t\sigma} \leq 2e^{-t^2/2}.$
\end{lem}

\ifnum\epsdeltastuff=1
Another simple corollary of this fact is the following.
\begin{coro}\label{cor:hanson-wright-extreme-lb}
    Let $X \sim \cN (0, 1) \in \R^d$ and let $A$ be a $d \times d$ PSD matrix.
    Then, there exists a universal constant $C \leq 10000$ so that
    \[
        \Pr \left[ X^\top A X < \frac{1}{C} \tr (A) \right] < \frac{1}{10} \; .
    \]
\end{coro}

\noindent The proof of this corollary appears in Section~\ref{sec:hanson-wright-extreme-lb}.
\fi

\subsubsection{Deterministic Regularity Conditions for Gaussians}

We will rely on certain regularity properties of i.i.d.\ samples from a Gaussian.
These are standard concentration inequalities, and a reference for these facts is Section 4 of~\cite{DiakonikolasKKLMS16}.
\begin{fact}\label{fact:gaussian-facts}
    Let $X_1, \ldots, X_n \sim \cN (0, \Sigma)$ i.i.d.\ for
    $\kappa_1 \id \preceq \Sigma \preceq \kappa_2 \id$.
    Let $Y_i = \Sigma^{-1/2} X_i$ and let
    \[
        \Sigmahat_Y = \frac{1}{n} \sum_{i = 1}^n Y_i Y_i^\top
    \]
    Then for every $\beta > 0$, the following conditions
    hold except with probability $1-O(\beta)$.
    \begin{align}
        &\forall i \in [n]~~~\| Y_i \|_2^2 \leq O\left( d \log (n / \beta) \right) \label{eq:cov-cond1} \\
        &\left( 1 - O \left( \sqrt{\frac{d + \log(1/\beta)}{n}} \right) \right) \cdot \id \preceq  \Sigmahat_Y \preceq \left( 1 + O \left( \sqrt{\frac{d + \log(1/\beta)}{n}} \right) \right) \cdot \id \label{eq:cov-cond2} \\
        &\left\| \id - \Sigmahat_Y \right\|_F \leq O \left( \sqrt{\frac{d^2 + \log(1/\beta)}{n}} \right) \label{eq:cov-cond3}
    \end{align}
\end{fact}

\noindent We now note some simple consequences of these conditions.
These inequalities follow from simple linear algebra and we omit their proof for conciseness.
\begin{lem}
\label{lem:rotations-cov-conds}
Let $Y_1, \ldots, Y_n$ satisfy \eqref{eq:cov-cond1}--\eqref{eq:cov-cond3}.
Fix $M \succ 0$, and for all $i = 1, \ldots, n$, let $Z_i = M^{1/2} Y_i$, and let $\wh\Sigma_{Z} = \frac{1}{n} \sum_{i=1}^{n} Z_i Z_i^\top$.  Let $\kappa'$ be the top eigenvalue of $M$.
Then
\begin{align*}
&\forall i \in [n]~~~\| Z_i \|_2^2 \leq O\left(\kappa' d \log (n / \beta) \right) \\
&\left( 1 - O \left( \sqrt{\frac{d + \log(1/\beta)}{n}} \right) \right) \cdot M \preceq \wh\Sigma_{Z} \preceq \left( 1 + O \left( \sqrt{\frac{d + \log(1/\beta)}{n}} \right) \right) \cdot M \\
&\left\| M - \wh\Sigma_Z \right\|_M \leq O \left( \sqrt{\frac{d^2 + \log(1/\beta)}{n}} \right)
\end{align*}
\end{lem}

\section{Eigenvalue Estimation}\label{sec:eigenvalues}

In this section, we present an algorithm that estimates the eigenvalues of a covariance matrix of a Gaussian distribution up to a constant factor, under the constraint of approximate differential privacy. This algorithm's function is important for the following sections, since it helps us overcome the issue that we have no prior bounds on the eigenvalues, as well as identify gaps between them. The algorithm performs a subsample-and-aggregate process. The samples are split into $t$ subsets and for each of them, the eigenvalues of the empirical covariance are computed. Denoting the $i$-th eigenvalue (in decreasing order of magnitude) of the $j$-th subsample by $\lambda_i^j$, for each $i$, we construct stability-based histograms and output an estimate of $\lambda_i$ based on the bucket where $\lambda_i^j$ tend to concentrate most.

\begin{algorithm}[H]
\label{alg:eigenvalues}
\caption{Differentially Private
    $\DPEE_{\eps, \delta, \beta}(X)$}
\KwIn{Samples $X_1,\dots,X_{n} \in \R^d$.
    Parameters $\eps, \delta, \beta > 0$.}
\KwOut{Noisy eigenvalues of $X$: $(\hat{\lambda}_1,\dots,\hat{\lambda}_d) \in \R^{d}$.}
\vspace{5pt}

Set parameters:
    $t \gets \tfrac{C_1\log(d/\delta\beta)}{\eps}$ \qquad $m \gets \lfloor n/t \rfloor$
\vspace{5pt}

Split $X$ into $t$ datasets of size $m$: $X^1,\dots,X^t$.
\vspace{5pt}

\tcp{Estimate the eigenvalues via DP Histograms.}
\For{$i \gets 1,\dots,d$}{
    \For{$j \gets 1,\dots,t$}{
        Let $\lambda_i^j$ be the $i$-th eigenvalue of $\tfrac{1}{m}\cdot X^{j \top}X^j$.
    }
    Divide $[0,\infty)$ into $\Omega\gets
        \{\dots,[1/\sqrt{2},1/2^{1/4})[1/2^{1/4},1)[1,2^{1/4}),[2^{1/4},\sqrt{2}),\dots\}
        \cup \{[0,0]\}$.\\
    Run $\left(\tfrac{\eps}{\sqrt{6 d \log(1/\delta)}},\tfrac{\delta}{d + 1}\right)$-DP histogram
        on all $\lambda_i^j$ over $\Omega$.\\
    \If{no bucket is returned}{
        \Return $\bot$.
    }
    Let $[l,r]$ be a non-empty bucket returned.\\
    Set $\bar{\lambda}_i \gets l$.
}
\vspace{5pt}

Sort $(\bar{\lambda}_1,\dots,\bar{\lambda}_d)$ to get $\hat{\lambda}_1,\dots,\hat{\lambda}_d$.
\vspace{5pt}

\Return $(\hat{\lambda}_1,\dots,\hat{\lambda}_d)$
\vspace{5pt}
\end{algorithm}

\begin{thm}\label{thm:eigenvalues}
    For every $\eps,\delta,\beta > 0$, there
    exists an $(\eps,\delta)$-DP algorithm, that takes
    $$n={O}\left(\frac{d^{3/2}\cdot\polylog(d,1/\delta,1/\eps,1/\beta)}{\eps}\right)$$
    samples from $\cN(0,\Sigma)$, for an arbitrary symmetric, positive-semidefinite $\Sigma \in \R^{d\times d}$, and outputs $\hat{\lambda}_1\geq\dots\geq\hat{\lambda}_d$,
    such that with probability at least $1-O(\beta)$,
    $\hat{\lambda}_i \in \left[\tfrac{\lambda_i(\Sigma)}{\sqrt{2}},\sqrt{2}\lambda_i(\Sigma)\right]$ for all $i$.
\end{thm}
\begin{proof}
    We show this by proving privacy and accuracy guarantees
    of Algorithm~\ref{alg:eigenvalues}.

    Fix an $i \in [d]$. Then by changing one sample in $X$,
    only one subsample of $X$ (say, $X^j$) gets changed, hence,
    only one $\lambda_i^j$ gets affected. This can change at most
    two histogram buckets, leading to sensitivity $2$. Therefore,
    by the privacy of private histograms Lemma~\ref{lem:priv-hist}, we have
    $\left(O\left(\tfrac{\eps}{\sqrt{d \log(1/\delta)}}\right), O\left(\tfrac{\delta}{d}\right)\right)$-DP
    for this fixed $i$. Applying Lemma~\ref{lem:dpcomp} gives us
    the final privacy guarantee.

    Now, we move on to the accuracy guarantees. It is sufficient
    to show that with probability at least $1-O(\beta/d)$, for each
    $1 \leq i \leq d$,
    $\bar{\lambda}_{i} \in \left[\tfrac{\lambda_i(\Sigma)}{\sqrt{2}},\sqrt{2}\lambda_i(\Sigma)\right]$.
    Fix an $i$. Now, by Lemma~\ref{lem:rotations-cov-conds}, with
    probability at least $1-O(\beta/d)$, the non-private estimates
    of $\lambda_i(\Sigma)$ must be within a factor of $2^{1/8}$
    of $\lambda_i(\Sigma)$ due to our sample complexity. Therefore,
    at most two consecutive buckets would be filled with $\lambda_i^j$'s.
    Due to our sample complexity and Lemma~\ref{lem:priv-hist}, those buckets
    are released with probability at least $1-O(\beta/d)$.
    Since they are built at a multiplicative width of $2^{1/4}$,
    they approximate the non-private estimate to within a factor
    of $2^{1/4}$. Therefore, the total multiplicative error is
    at most a factor of $2$. Taking the union bound over all $i$,
    we get the required result.
\end{proof}

\section{Subspace Recovery}\label{sec:subspace}

We improve the guarantees of the subspace algorithm
from \cite{SinghalS21} for our problem, where we are
willing to pay $\poly(d)$ in the sample complexity.
In our version, the algorithm's aggregation step uses the ball-finding
algorithm from \cite{NissimSV16}, followed by noisy mean estimation, instead of using
high-dimensional stability-based histograms as in \cite{SinghalS21}.
For completeness, we restate the entire algorithm,
but just point out the differences in the proof of
the final accuracy lemma from \cite{SinghalS21}.

\begin{algorithm}[H]
\caption{\label{alg:subspace}DP Subspace Estimator
    $\subspace_{\eps, \delta, \alpha, \gamma, k}(X)$}
\KwIn{Samples $X_1,\dots,X_{n} \in \R^d$.
    Parameters $\eps, \delta, \alpha, \gamma, k > 0$.}
\KwOut{Projection matrix $\wh{\Pi} \in \R^{d \times d}$ of rank $k$.}
\vspace{5pt}

Set parameters:
    $t \gets \tfrac{C_0\sqrt{dk}\cdot\polylog(d,k,\frac{1}{\eps},\frac{1}{\delta})}{\eps}$ \qquad
    $m \gets \lfloor n/t \rfloor$ \qquad $q \gets C_1 k$\\
    \qquad \qquad \qquad $r \gets \tfrac{C_2\gamma\sqrt{d}(\sqrt{k}+\sqrt{\ln(kt)})}{\sqrt{m}}$
\vspace{5pt}

Sample reference points $p_1,\dots,p_q$ from $\cN(\vec{0},\id)$ independently.
\vspace{5pt}

\tcp{Subsample from $X$, and form projection matrices.}
\For{$j \in 1,\dots,t$}{
    Let $X^j = (X_{(j-1)m+1},\dots,X_{jm}) \in \mathbb{R}^{d \times m}$.\\
    Let $\Pi_j \in \mathbb{R}^{d \times d}$ be the projection matrix onto the subspace spanned by the eigenvectors of $X^j (X^j)^{\top} \in \mathbb{R}^{d \times d}$ corresponding to the largest $k$ eigenvalues.\\
    \For{$i \in 1,\dots,q$}{
        $p_{i}^j \gets \Pi_j p_i$
    }
}
\vspace{5pt}

\tcp{Aggregate using a ball-finding algorithm.}
\For{$i \in [q]$}{
    Let $P_i \in \RR^{d \times t}$ be the dataset, where column $j$ is $p_i^j$.\\
    Set $c_i \gets \goodcenter_{\frac{\eps}{\sqrt{q\ln(1/\delta)}},\frac{\delta}{q},r}(P_i)$.\\
}
Set $R \gets C_3r\sqrt{\log(t)}$
\vspace{5pt}

\tcp{Return the subspace.}
    Let $\sigma \gets \tfrac{4R\sqrt{q}\ln(q/\delta)}{\eps t}$.\\
\For{each $i \in [q]$}{
    Truncate all $p_i^j$'s to within $\ball{c_i}{R}$.\\
    Let $\wh{p}_i \gets \sum\limits_{j=1}^{t}{p_i^j} + \cN(0,\sigma^2\id_{d \times d})$.
}
Let $\wh{P} \gets (\wh{p}_i,\dots,\wh{p}_q)$.\\
Let $\wh{\Pi}$ be the projection matrix of the top-$k$ subspace of $\wh{P}$.\\
\Return $\wh{\Pi}.$
\vspace{5pt}
\end{algorithm}

\begin{lem}
    Algorithm~\ref{alg:subspace} is $(2\eps,2\delta)$-DP.
\end{lem}
\begin{proof}
    The first aggregation step of finding $c_i$ is
    $(\eps,\delta)$-DP by Theorem~\ref{thm:goodcenter} and
    Lemma~\ref{lem:dpcomp}.
    In the mean estimation step, because we are restricting
    all the $p_i^j$'s to within $\ball{c_i}{R}$, the sensitivity
    is $2R$, since by changing one point in $X$, we can
    change exactly one $p_i^j$ by $2R$ in $\ell_2$ norm.
    Therefore, by Lemmata~\ref{lem:gaussiandp} and \ref{lem:dpcomp}, this step
    is $(\eps,\delta)$-DP. The final privacy guarantee
    follows from Lemma~\ref{lem:dpcomp}.
\end{proof}

\begin{lem}[Lemma~4.9 of \cite{SinghalS21} Modified]\label{lem:subspace}
    Let $\wh{\Pi}$ be the projection matrix as defined in
    Algorithm~\ref{alg:subspace}, $n$ be the total
    number of samples, and $0 < \psi < 1$. If
    $$t \geq O\left(\frac{\sqrt{dk}\cdot\polylog(d,k,\frac{1}{\eps},\frac{1}{\delta})}{\eps}\right)
    ~~~\text{and}~~~
    m \geq O\left(\frac{d\cdot
        \polylog(d,k,\frac{1}{\eps},\frac{1}{\delta})}{\psi^2}\right),$$
    which implies that
    $$n \geq O\left(\frac{d^{1.5} \sqrt{k}\cdot\polylog(d,k,\frac{1}{\eps},\frac{1}{\delta})}
        {\eps\psi^2}\right),$$
    then $\|\Pi-\wh{\Pi}\| \leq \psi\gamma$ with probability
    at least $0.7$.
\end{lem}
\begin{proof}
    For each $i \in [q]$, let $p_i^*$ be the projection
    of $p_i$ on to the subspace spanned by $\Sigma_k$,
    $\wh{p}_i$ be as defined in the algorithm, and $p_i^j$
    be the projection of $p_i$ on to the subspace spanned
    by the $j^{\mathrm{th}}$ subset of $X$. From the analysis
    in \cite{SinghalS21}, we know that for a fixed $i$, all $p_i^j$'s
    are contained in a ball of radius $r$. Therefore, all points
    in $P_i$ lie in a ball of radius $r$. Therefore, by the
    guarantees of $\goodcenter$ (Theorem~\ref{thm:goodcenter}),
    $\ball{c_i}{R}$ contains all
    of $p_i^j$'s, such that $R \in O(r\sqrt{\ln(t)})$. This implies
    that $p_i^*$ is also contained within $\ball{c_i}{R}$.

    Now, let $P = (p_1^*,\dots,p_q^*)$.
    Suppose $\wh{P}=(\wh{p}_1,\dots,\wh{p}_q)$ as defined in the algorithm.
    Then by above, $\wh{P}=P+E$ for some $E \in \RR^{d\times q}$.
    The goal is to show that
    $\|\Pi-\wh{\Pi}\| \leq O(\tfrac{\|E\|}{\sqrt{k}}) \leq O(\gamma\psi)$.
    We set $E = E_0 + E_1$, where $E_0$ is the sampling error,
    and $E_1$ is the error due to privacy, In other words, let
    $\wb{p}_i = \tfrac{1}{t}\sum\limits_{j=1}^{t}{p_i^j}$ and
    $\wb{P} = (\wb{p}_1,\dots,\wb{p}_q)$; then $E_0 = \wb{P} - P$
    and $E_1 = \wh{P} - \wb{P}$.

    We first analyse $\|E_0\|$. Let $\Pi^j$ be the subspace spanned
    by the $j$-th subsample. We know that the subspaces spanned by
    $P^j = (p_1^j,\dots,p_q^j)$ and the $j$-th subsample are the same.
    Therefore, $\|\Pi-\Pi^j\| \in \Theta(\tfrac{\|P^j-P\|}{\sqrt{k}})
    \leq \gamma\sqrt{\tfrac{d}{m}}$ by Lemmata~2.4 and~4.5, and
    Corollary~2.7 of \cite{SinghalS21}. Therefore,
    \begin{align*}
        \frac{\|E_0\|}{\sqrt{k}} &\leq O\left(\frac{\|\wb{P}-P\|}{\sqrt{k}}\right)\\
            &= O\left(\frac{\|\frac{1}{t}\sum\limits_{j=1}^{t}{P^j}-P\|}{\sqrt{k}}\right)\\
            &\leq O\left(\frac{\frac{1}{t}\sum\limits_{j=1}^{t}\|{P^j}-P\|}{\sqrt{k}}\right)\\
            &\leq O\left(\frac{1}{t}\cdot\sum\limits_{j=1}^{t}{\gamma\sqrt{\frac{d}{m}}}\right)\\
            &\leq O\left(\gamma\sqrt{\frac{d}{m}}\right)\\
            &\in O(\gamma\psi). \tag{By our sample complexity.}
    \end{align*}

    Next, we analyse $\|E_1\|$. $E_1$ is a matrix with
    i.i.d. entries from $\cN(0,\sigma^2)$. Therefore, by
    Lemma~2.4 of \cite{SinghalS21}, we have
    \begin{align*}
        \frac{\|E_1\|}{\sqrt{k}} &\in O\left(\frac{\sigma\sqrt{d}}{\sqrt{k}}\right)\\
            &\in O\left(\frac{r\sqrt{\log(t)kd\log(k/\delta)}}{\eps t\sqrt{k}}\right)\\
            &\in O\left(\frac{r}{\sqrt{k}}\right) \tag{By our sample complexity.}\\
            &\in O\left(\gamma\sqrt{\frac{d}{m}}\right)\\
            &\in O(\gamma\psi). \tag{By our sample complexity.}
    \end{align*}

    Therefore, we have $\|E\| \in O(\gamma\psi)$.

    Let $E=E_P+E_{\wb{P}}$,
    where $E_P$ is the component of $E$ in the subspace
    spanned by $P$, and $E_{\wb{P}}$ be the orthogonal
    component. Let $P' = P + E_P$. We will be analysing
    $\wh{P}$ with respect to $P'$.

    As before, we will try to
    bound the distance between the subspaces spanned
    by $P'$ and $\wh{P}$.
    The quantities $a,z_{12}$ remain unchanged, but $b,z_{21}$
    change.
    \begin{align*}
        b &\leq \|E_{\wb{P}}\|\\
        z_{21} &\leq \|E_{\wb{P}}\|\\
    \end{align*}

    Therefore, we get the final error:
    \begin{align*}
        \|\Pi-\wh{\Pi}\| &\leq \frac{az_{21} + bz_{12}}
                {a^2-b^2-\min\{z_{12}^2,z_{21}^2\}}\\
            &\leq \gamma\psi.
    \end{align*}

    This completes our proof.
\end{proof}

This gives us the following theorem about Algorithm~\ref{alg:subspace}.

\begin{thm}\label{thm:subspace}
    Let $\Sigma \in \RR^{d \times d}$ be a symmetric,
    PSD matrix, such that for $1 \leq k < d$ and $\gamma < 1$,
    $\tfrac{\lambda_{k+1}(\Sigma)}{\lambda_k(\Sigma)} < \gamma^2$.
    Suppose $\Pi$ is the subspace spanned by the top $k$
    eigenvectors of $\Sigma$.
    Then for all $\eps,\delta,\beta,\psi > 0$, there exists an
    $(\eps,\delta)$-DP algorithm, that takes
    $$n \geq O\left(\frac{d^{1.5} \sqrt{k}\cdot\polylog(d,k,\frac{1}{\eps},\frac{1}{\delta},\frac{1}{\beta})}
        {\eps\psi^2}\right)$$
    samples from $\cN(0,\Sigma)$, and outputs a projection
    matrix $\wh{\Pi}$, such that with probability at least
    $1-O(\beta)$, $\|\Pi-\wh{\Pi}\| \leq \psi\gamma$.
\end{thm}
\begin{proof}
    The claim, but with error probability $0.35$, is guaranteed
    from Lemma~\ref{lem:subspace}. Now, we just have to boost the
    success probability. This can be done using Theorem~4.10 of
    \cite{SinghalS21}.
\end{proof}

\section{Private Preconditioning}\label{sec:preconditioner}

In this section, we develop a preconditioning technique that does not rely on knowledge of a priori bounds on the eigenvalues of the covariance matrix of the underlying distribution. It is the main preprocessing step that makes the Gaussian covariance almost spherical. For the following, we assume that the eigenvalues of the covariance matrix $\Sigma$ are examined in non-increasing order $\lambda_1 \geq \dots \geq \lambda_d > 0$.

\subsection{Coarse Preconditioning}
\label{subsec:coarse}

We describe here the function of the ``coarse'' preconditioner which, along with Algorithm~\ref{alg:subspace}, constitutes the main technical novelty of our approach. The purpose served by this subroutine is to reduce gaps between consecutive eigenvalues (say $\lambda_k$ and $\lambda_{k + 1}$). Observe that, our only assumptions are that the ratio $\frac{\lambda_{k + 1}(\Sigma)}{\lambda_k(\Sigma)}$ is below some threshold and that the eigenvalues that come before exhibit no significant gaps ($\frac{\lambda_k}{\lambda_1}$ is lower bounded appropriately, implying that it is larger than some absolute constant). The first condition essentially prohibits us from using the preconditioning technique from~\cite{KamathLSU19}, since we do not know how large the gap between $\lambda_k$ and $\lambda_{k + 1}$ may be. Instead, the algorithm uses our adaptation of the subspace algorithm of~\cite{SinghalS21} (see Algorithm~\ref{alg:subspace}) in order to approximate the subspace that corresponds to the eigenvalues that come before the gap. 
Specifically, we obtain projection matrices $\Pi_V$ onto a subspace $V$ and $\Pi_{V^{\perp}} = \mathbb{I} - \Pi_V$ onto its complement $V^{\perp}$, such that these matrices are close in spectral norm to the projections onto the top $k$ eigenspace of $\Sigma$ and its complement. Rescaling our data by a matrix of the form $A = x \Pi_V + y \Pi_{V^{\perp}}$ roughly results in the eigenvalues of the covariance matrix corresponding to $V$ and $V^{\perp}$ being rescaled by $x^2$ and $y^2$, respectively. Setting the scalars $x$ and $y$ appropriately will reduce the eigenvalue gap, even if the subspace $V$ is not perfectly aligned with the top $k$ eigenvalues.
Interestingly, if the eigengap is large (i.e., the ratio  $\frac{\lambda_{k + 1}(\Sigma)}{\lambda_k(\Sigma)}$ is small), then our algorithm works just as well as when it is small. This is because the subspace recovery subroutine will become more accurate in this setting as it outputs a projection matrix, whose error scales with this gap.
Note that this step reduces the eigengap to a large extent, but does not exactly get us in the range that we would desire, that is, the gap between the $1$-st and the $(k+1)$-th eigenvalues is greatly reduced, but it is still not small enough to maintain the loop invariant of Algorithm~\ref{alg:dpp}, which says that in iteration $i$, the gap between the $1$-st and the $i$-th eigenvalues is bounded. We address this issue in Section~\ref{sec:fine}.

Having described the algorithm above, we now present the corresponding pseudocode, followed by its analysis.

\bigskip

\begin{algorithm}[H]
\label{alg:dpcp}
\caption{Differentially Private
    $\DPCP_{\eps, \delta, \beta, k, \hat{\gamma}}(X)$}
\KwIn{Samples $X_1,\dots,X_{n} \in \R^d$.
    Parameters $\eps, \delta, \beta, k > 0, \hat{\gamma}\geq0$.}
\KwOut{Matrix $A \in \R^{d \times d}$.}
\vspace{5pt}

Set $1-\eta \gets \hat{\gamma}$.\\
Set $\hat{\Pi}_{1:k} \gets \subspace_{\eps,\delta,\beta,k,\hat{\gamma}}(X)$
    and $\hat{\Pi}_{k+1:d} \gets \id - \hat{\Pi}_{1:k}$.\\
Set $A \gets (1-\eta)\hat{\Pi}_{1:k} + \hat{\Pi}_{k+1:d}$.\\
\vspace{5pt}

\Return $A$.
\vspace{5pt}
\end{algorithm}

\begin{thm}[Coarse Preconditioner]\label{thm:dpcp}
Let $0 < \wb{\gamma} \leq 1$ and $0 < \hat\gamma < 1$ be
arbitrary parameters.
 Then for all $\eps,\delta,\beta>0$ and $$n \geq O\left(\frac{d^{2}\cdot\polylog(d,\frac{1}{\eps},\frac{1}{\delta},\frac{1}{\beta})} {\eps\wb{\gamma}^4}\right),$$ there exists an $(\eps,\delta)$-DP algorithm, such that the following holds.
Let $X=(X_1,\dots,X_n)$ be i.i.d.~samples from $\cN(0,\Sigma)$, where, for some $1 \leq k < d$, $\tfrac{\lambda_k(\Sigma)}{\lambda_1(\Sigma)} \geq \wb{\gamma}^2$, and $\gamma^2 := \tfrac{\lambda_{k+1}(\Sigma)}{\lambda_{k}(\Sigma)} \in \left[\tfrac{\hat\gamma^2}{4},4\hat\gamma^2\right]$.
Then with probability at least $1-O(\beta)$, the algorithm takes $X$ and $\hat{\gamma}$ as input, and outputs $A \in \mathbb{R}^{d \times d}$ that satisfies $\tfrac{\lambda_{k+1}(A\Sigma A)}{\lambda_1(A\Sigma A)} \geq \tfrac{\wb{\gamma}^2}{40}$.
\end{thm}
\begin{proof}
We prove the privacy and accuracy guarantees of Algorithm~\ref{alg:dpcp}. Privacy follows from the privacy guarantees of $\subspace$ (Theorem~\ref{thm:subspace}) and post-processing of DP (Lemma~\ref{lem:postprocessing}).

Now, we prove the accuracy guarantees. Suppose $\Sigma = U \Lambda U^{\top}$ and $U, \Lambda, \Sigma \in \mathbb{R}^{d \times d}$, where $U^{\top}U=I$ and $\Lambda$ is diagonal with entries $\lambda_1 \ge \lambda_2 \ge \cdots \ge \lambda_d \ge 0$.

We know that there is a large eigengap -- i.e., $\lambda_{k+1} = \gamma^2 \cdot \lambda_k$ for some $k \in [d]$ and $0 < \gamma \ll 1$. Consider the subspace spanned by the eigenvectors corresponding to $\lambda_1, \dots, \lambda_k$ and let $\Pi_{1:k}$ be the corresponding projection matrix. We then run the subspace algorithm $\subspace$ \cite{SinghalS21} with parameters $\eps,\delta,\beta,k,\hat{\gamma}$ to obtain $\hat{\Pi}_{1:k} \in \mathbb{R}^{d \times d}$ satisfying $\|\hat{\Pi}_{1:k} - \Pi_{1:k}\| \le \phi \iff \|\hat{\Pi}_{k+1:d} - \Pi_{k+1:d}\| \le \phi$ with probability at least $1-O(\beta)$, where, because of our sample complexity $\phi \le \tfrac{\hat{\gamma}\wb{\gamma}^2}{100}$.

Now let $y_i = [(1-\eta)\hat{\Pi}_{1:k} + \hat{\Pi}_{k+1:d}] X_i$ for all $i \in [n]$. Here, $0 \le \eta = 1-\hat{\gamma}$. Then $y_1, \cdots, y_n \in \mathbb{R}^d$ are $n$ independent draws from $\cN(0,\hat\Sigma)$, where
\[
    \hat\Sigma = (\hat{\gamma} \hat{\Pi}_{1:k} + \hat{\Pi}_{k+1:d}) \Sigma (\hat{\gamma}\hat{\Pi}_{1:k} + \hat{\Pi}_{k+1:d}).
\]
We set $\xi = \xi_1 = \hat{\Pi}_{1:k} - \Pi_{1:k}$ and $\xi_2 = \hat{\Pi}_{k + 1:d} - \Pi_{k + 1:d} = - \xi$ where $\|\xi\| \le \phi$. We have for $\hat{\Sigma}$:
\begin{align*}
    \hat\Sigma &= (\hat{\gamma} \hat{\Pi}_{1:k} + \hat{\Pi}_{k+1:d}) \Sigma (\hat{\gamma} \hat{\Pi}_{1:k} + \hat{\Pi}_{k+1:d}) \\
               &= (\hat{\gamma} \xi_1 + \xi_2 + \hat{\gamma} \Pi_{1:k} + \Pi_{k+1:d}) \Sigma (\hat{\gamma} \xi_1 + \xi_2 + \hat{\gamma} \Pi_{1:k} + \Pi_{k+1:d}) \\
               &= (\hat{\gamma} \xi_1 + \xi_2 + \hat{\gamma} \Pi_{1:k} + \Pi_{k+1:d}) \Sigma (\hat{\gamma} \xi_1 + \xi_2 + \hat{\gamma} \Pi_{1:k} + \Pi_{k+1:d}) \\
               &= (\hat{\gamma} \xi_1 + \xi_2) \Sigma (\hat{\gamma} \xi_1 + \xi_2) + (\hat{\gamma} \Pi_{1:k} + \Pi_{k+1:d}) \Sigma (\hat{\gamma} \xi_1 + \xi_2) \\
               &+ (\hat{\gamma} \xi_1 + \xi_2) \Sigma (\hat{\gamma} \Pi_{1:k} + \Pi_{k+1:d}) + (\hat{\gamma} \Pi_{1:k} + \Pi_{k+1:d}) \Sigma (\hat{\gamma} \Pi_{1:k} + \Pi_{k+1:d}) \\
               &= (\hat{\gamma} \xi_1 + \xi_2) \Sigma (\hat{\gamma} \xi_1 + \xi_2) + \hat{\gamma} \Pi_{1:k} \Sigma (\hat{\gamma} \xi_1 + \xi_2) + \Pi_{k+1:d} \Sigma (\hat{\gamma} \xi_1 + \xi_2) \\
               &+ \hat{\gamma} (\hat{\gamma} \xi_1 + \xi_2) \Sigma \Pi_{1:k} + (\hat{\gamma} \xi_1 + \xi_2) \Sigma \Pi_{k+1:d} + \hat{\gamma}^2 \Pi_{1:k} \Sigma \Pi_{1:k} + \Pi_{k+1:d} \Sigma \Pi_{k+1:d}.
\end{align*}
Now, we need to find an upper limit for $\lambda_1(\hat{\Sigma})$, and a lower limit for $\lambda_{k+1}(\hat{\Sigma})$.
    
We start with the upper bound on $\lambda_1(\hat{\Sigma})$.
\begin{align*}
    \left\|\hat{\Sigma}\right\| &\le \left\|\hat{\gamma} \xi_1 + \xi_2\right\|^2 \left\|\Sigma\right\| + 2 \hat{\gamma} \left\| \Pi_{1:k} \Sigma \right\| \left\| \hat{\gamma} \xi_1 + \xi_2 \right\| + 2 \left\| \Pi_{k+1:d} \Sigma \right\| \left\| \hat{\gamma} \xi_1 + \xi_2 \right\| \\
                                &~~~+ \hat{\gamma}^2 \left\|\Pi_{1:k} \Sigma \Pi_{1:k}\right\| + \left\|\Pi_{k+1:d} \Sigma \Pi_{k+1:d}\right\| \\
                                &\le (1 - \hat{\gamma})^2 \tfrac{\hat{\gamma}^2 \wb{\gamma}^4}{10000} \lambda_1(\Sigma) + 2 (1 - \hat{\gamma}) \tfrac{\hat{\gamma}^2 \wb{\gamma}^2}{100} \lambda_1(\Sigma) + 2 (1 - \hat{\gamma}) \tfrac{\hat{\gamma} \wb{\gamma}^2}{100} \lambda_{k + 1}(\Sigma) \\
                                &~~~+ \hat{\gamma}^2 \lambda_1(\Sigma) + \lambda_{k + 1}(\Sigma) \\
                                &\le \frac{\gamma^2 \wb{\gamma}^2}{2500} \lambda_k(\Sigma) + \frac{2 \gamma^2}{25} \lambda_k(\Sigma) + 2 (1 - \hat{\gamma}) \tfrac{\hat{\gamma} \wb{\gamma}^2}{100} \lambda_{k + 1}(\Sigma) + \frac{4 \gamma^2}{\wb{\gamma}^2} \lambda_k(\Sigma) + \lambda_{k + 1}(\Sigma) \\
                                &\le \frac{\wb{\gamma}^2}{2500} \lambda_{k + 1}(\Sigma) + \frac{2}{25} \lambda_{k + 1}(\Sigma) + 2 (1 - \hat{\gamma}) \tfrac{\hat{\gamma} \wb{\gamma}^2}{100} \lambda_{k + 1}(\Sigma) + \frac{4}{\wb{\gamma}^2} \lambda_{k + 1}(\Sigma) + \lambda_{k + 1}(\Sigma) \\
                                &\le \frac{5}{\wb{\gamma}^2} \lambda_{k + 1}(\Sigma).
\end{align*}
Now, we prove a lower bound on $\lambda_{k+1}(\hat{\Sigma})$.
\begin{align*}
    \lambda_{k+1}(\hat{\Sigma}) &\geq \lambda_{k+1}\left(\hat{\gamma}^2 \Pi_{1:k} \Sigma \Pi_{1:k} + \Pi_{k+1:d} \Sigma \Pi_{k+1:d}\right) \\
            &~~~ + (1 - \hat{\gamma})^2 \lambda_d\left(\xi \Sigma \xi \right) -
                \hat{\gamma} (1 - \hat{\gamma}) \lambda_d\left( \Pi_{1:k} \Sigma \xi \right) \\
            &~~~ - (1 - \hat{\gamma}) \lambda_d\left(\Pi_{k+1:d} \Sigma \xi \right) - \hat{\gamma} (1 - \hat{\gamma})
                \lambda_d\left(\xi \Sigma \Pi_{1:k}\right) \\
            &~~~ - (1 - \hat{\gamma}) \lambda_d\left(\xi \Sigma \Pi_{k+1:d}\right)
                \tag{Lemma~\ref{lem:weyl}}\\
            &\geq \frac{\lambda_{k+1}}{4} - 2 \hat{\gamma} (1 - \hat{\gamma}) \left\|\xi\right\| \left\| \Pi_{1:k} \Sigma\right\| - 2 (1 - \hat{\gamma}) \left\|\xi\right\| \left\|\Sigma \Pi_{k+1:d}\right\| \\
            &\geq \frac{\lambda_{k+1}}{4} - \frac{\hat{\gamma}^2 \wb{\gamma}^2}{50} \lambda_1(\Sigma) - \frac{\hat{\gamma} \wb{\gamma}^2}{50} \lambda_{k + 1}(\Sigma) \\
            &\geq \frac{\lambda_{k+1}}{4} - \frac{2}{25} \lambda_{k + 1}(\Sigma) - \frac{\hat{\gamma} \wb{\gamma}^2}{50} \lambda_{k + 1}(\Sigma) \\
            &\geq \frac{\lambda_{k+1}}{8}
    \end{align*}

    Therefore, $\tfrac{\lambda_{k+1}(\hat{\Sigma})}{\lambda_1(\hat{\Sigma})}
    \geq \tfrac{\wb{\gamma}^2}{40}$.
\end{proof}

\subsection{Fine Preconditioning}
\label{sec:fine}
In this section, we present our second preconditioning constituent (the ``fine" preconditioner) that is used in the presence of small cumulative gaps. This component of our preconditioning process is similar to the one that appears in~\cite{KamathLSU19}. It first uses the naive estimator (i.e., clipping data based on the covariance matrix's spectrum and noising the empirical covariance, Algorithm~\ref{alg:naive}) to get a rough estimate of the covariance.
This gives us enough information about the top $k+1$ eigenvectors and eigenvalues to operate (approximately) within the top-$(k+1)$ subspace, allowing us to shrink down the top $k$ eigenvalues by a small multiplicative factor. We initially assume that the gap between the $1$-st and the $(k+1)$-th eigenvalues is large, but not too large, essentially the setting that we will be in after running the coarse preconditioner described in Section~\ref{subsec:coarse}. In other words, when the gap between the $1$-st and the $(k+1)$-th eigenvalues is loosely bounded, the fine preconditioner tightens that gap. We now present our algorithm and its analysis.

\bigskip

\begin{algorithm}[H]
\label{alg:dpfp}
\caption{Differentially Private
    $\DPFP_{\eps, \delta, \beta, k, \wb{\gamma}, \kappa}(X)$}
\KwIn{Samples $X_1,\dots,X_{n} \in \R^d$.
    Parameters $\eps, \delta, \beta, k, \wb{\gamma}, \kappa > 0$.}
\KwOut{Matrix $A \in \R^{d \times d}$.}
\vspace{5pt}

Set $Z \gets \NE_{\eps, \delta, \beta, \kappa}(X)$.\\
Let $S \gets \{i : \lambda_i(Z) \geq \tfrac{\lambda_{k+1}(Z)}{16\wb{\gamma}^2}\}$.\\
Let $g_i \gets \sqrt{\frac{\lambda_i(Z)}{\lambda_{k+1}(Z)}}$.\\
Let $v_i$ be the $i$-th eigenvector of $Z$.\\
Set $\hat{\Pi}_S \gets
    \sum\limits_{i\in S}{\frac{v_i v_i^{\top}}{4 g_i\wb{\gamma}}}$
    and $\hat{\Pi}_{\wb{S}} \gets \sum\limits_{i \not\in S}{v_i v_i^{\top}}$.\\
Set $A \gets \hat{\Pi}_S + \hat{\Pi}_{\wb{S}}$.
\vspace{5pt}

\Return $A$.

\vspace{5pt}
\end{algorithm}

\begin{thm}[Fine Preconditioner]\label{thm:dpfp}
    Let $X=(X_1,\dots,X_n)$ be i.i.d. samples from $\cN(0,\Sigma)$,
    such that for some $1 \leq k < d$,
    $\tfrac{\lambda_{k+1}(\Sigma)}{\lambda_1(\Sigma)} \geq \tau^2\wb{\gamma}^2$
    for $\wb{\gamma} \leq 1$.
    Then for all $\eps,\delta>0$, there exists
    an $(\eps,\delta)$-DP algorithm, such that if
    $$n \geq O\left(\frac{d^{3/2}\cdot
        \polylog(d,\frac{1}{\eps},\frac{1}{\delta},\frac{1}{\beta})}
        {\eps\tau^2\wb{\gamma}^2}\right),$$
    then with probability at least $1-O(\beta)$, it takes $X$ as input,
    and outputs a matrix $A$ that satisfies
    $\tfrac{\lambda_{k+1}(A\Sigma A)}{\lambda_1(A\Sigma A)} \geq \wb{\gamma}^2$.
\end{thm}
\begin{proof}
We prove the privacy and accuracy guarantees of Algorithm~\ref{alg:dpfp}. Privacy follows from the guarantees of Lemma~\ref{lem:dp-naive}.

Now, we prove the accuracy. Let $\hat{\Pi}_S$ and $\hat{\Pi}_{\wb{S}}$ be matrices as defined in Algorithm~\ref{alg:dpfp}. We first show an upper bound on $\|A\Sigma A\|$. For this, by Lemma~\ref{lem:rotations-cov-conds}, it is enough to prove an upper bound on $\|A(Z-N)A\|$.
\begin{align*}
    \|A(Z-N)A\| &\leq \|AZA\| + \|ANA\|\\
    &\leq \|\hat{\Pi}_S Z\hat{\Pi}_S + \hat{\Pi}_{\wb{S}} Z\hat{\Pi}_{\wb{S}}\| + \|N\| \\
    &\leq \frac{\lambda_{k+1}(Z)}{16\wb{\gamma}^2} + \frac{\lambda_{k+1}(Z)}{16\wb{\gamma}^2}\\
    &= \frac{\lambda_{k+1}(Z)}{8\wb{\gamma}^2}
\end{align*}
In the above, the third inequality comes from Corollary~\ref{coro:naive} and our sample complexity. This shows that $\|A\Sigma A\| \leq \tfrac{\lambda_{k+1(Z)}}{4\wb{\gamma}^2}$.

Now, we show a lower bound on $\lambda_{k+1}(A\Sigma A)$. As before, by Lemma~\ref{lem:rotations-cov-conds}, it is enough to show a lower bound on $\lambda_{k+1}(A(Z-N)A)$.
\begin{align*}
    \lambda_{k+1}(A(Z-N)A) &\geq \lambda_{k+1}(AZA) - \|ANA\| \tag{Lemma~\ref{lem:weyl}}\\
    &\geq \lambda_{k+1}(\hat{\Pi}_S Z\hat{\Pi}_S + \hat{\Pi}_{\wb{S}} Z\hat{\Pi}_{\wb{S}}) - \|N\|\\
    &\geq \lambda_{k+1}(Z) - \frac{\lambda_{k+1}(Z)}{2}\\
    &\geq \frac{\lambda_{k+1}(Z)}{2}
\end{align*}
In the above, the third inequality again follows from Corollary~\ref{coro:naive} and our sample complexity. This gives us $\lambda_{k+1}(A\Sigma A) \geq \tfrac{\lambda_{k+1}(Z)}{4}$.

Therefore, $\tfrac{\lambda_{k+1}(A \Sigma A)}{\lambda_1(A \Sigma A)} \geq \wb{\gamma}^2$.
\end{proof}

\subsection{Putting Everything Together}

We are now ready to present our overall preconditioning algorithm (Algorithm~\ref{alg:dpp}). The algorithm essentially relies on a \emph{dynamic programming} approach. In particular, the $i-$th iteration always starts under the assumption that the cumulative gap of the eigenvalues $\lambda_1 \geq \dots \geq \lambda_i$ is (relatively) small, so the focus is on the gaps involving the eigenvalue $\lambda_{i + 1}$, namely the ratios $\frac{\lambda_{i + 1}}{\lambda_i}$ and $\frac{\lambda_{i + 1}}{\lambda_1}$. Based on how small these ratios are, the algorithm may use either the coarse or the fine preconditioner, or both. Doing so, it ensures that, at the start of the next iteration, the loop's invariant will be preserved. At the end of a run of this algorithm, we get a linear transformation that reduces the multiplicative gap between the $1$-st and the $d$-th eigenvalues of $\Sigma$ to $\Omega(1)$. The algorithm and its analysis follow.

\bigskip

\begin{algorithm}[H]
\label{alg:dpp}
\caption{Differentially Private
    $\DPP_{\eps, \delta, \beta}(X)$}
\KwIn{Samples $X_1,\dots,X_{n} \in \R^d$.
    Parameters $\eps, \delta, \beta > 0$.}
\KwOut{Matrix $A \in \R^{d \times d}$.}
\vspace{5pt}

Set parameter: $\tau^2 \gets \frac{1}{10000}$
    \qquad $\wb{\gamma}^2 \gets \frac{40}{10000}$
\vspace{5pt}

Let $A \gets \id$.\\
$\hat{\lambda}_1,\dots,\hat{\lambda}_d \gets \DPEE_{\eps,\delta,\beta}(X)$.\\
Set $i \gets 1$.\\
\While{$i < d$}{
    \If{$\frac{\hat{\lambda}_{i+1}}{\hat{\lambda}_i} < 4\tau^2$}{
        $B \gets \DPCP_{\tfrac{\eps}{\sqrt{6 d \log(1/\delta)}},\tfrac{\delta}{d + 1},\frac{\beta}{d},i,
            \sqrt{\frac{\hat{\lambda}_{i+1}}{\hat{\lambda}_i}}}(X)$.\\
        $A \gets BA$.\\
        $X \gets AX$.\\
        $Z \gets \NE_{\tfrac{\eps}{\sqrt{6 d \log(1/\delta)}},\tfrac{\delta}{d + 1},
            \frac{\beta}{d}}(X)$.\\
        \If{$\frac{\lambda_{i+1}(Z)}{\lambda_1(Z)} < 4\wb{\gamma}^2$}{
            $C \gets \DPFP_{\tfrac{\eps}{\sqrt{6 d \log(1/\delta)}},\tfrac{\delta}{d + 1},\frac{\beta}{d},i,\wb{\gamma},{\lambda}_1(Z)}(X)$.\\
            $A \gets CA$.\\
            $X \gets AX$.
        }
    }
    \ElseIf{$\frac{\hat{\lambda}_{i+1}}{\hat{\lambda}_1} < 4\wb{\gamma}^2$}{
        $D \gets \DPFP_{\tfrac{\eps}{\sqrt{6 d \log(1/\delta)}},\tfrac{\delta}{d + 1},
            \frac{\beta}{d},i,\wb{\gamma},{\lambda_1}(Z)}(X)$.\\
        $A \gets DA$.\\
        $X \gets AX$.
    }
    $Z \gets \NE_{\tfrac{\eps}{\sqrt{6 d \log(1/\delta)}},\tfrac{\delta}{d + 1},\frac{\beta}{d}}(X)$.\\
    $\hat{\lambda}_1,\dots,\hat{\lambda}_d \gets
        \DPEE_{\tfrac{\eps}{\sqrt{6 d \log(1/\delta)}},\tfrac{\delta}{d + 1},\frac{\beta}{d}}(X)$.\\
    $i \gets i + 1$.
}
\vspace{5pt}

\Return $A$.
\vspace{5pt}
\end{algorithm}

\begin{thm}[DP Preconditioner]\label{thm:dpp}
    Let $\Sigma \in \R^{d \times d}$ be a symmetric, positive-definite matrix.  There exists
    an $(\eps,\delta)$-DP algorithm, such that if $X = (X_1,\dots,X_n) \sim \cN(0,\Sigma)$ and
    $$n \geq O\left(\frac{d^{2.5}\cdot
        \polylog(d,\frac{1}{\eps},\frac{1}{\delta},\frac{1}{\beta})}{\eps}\right),$$
    then with probability at least $1-\beta$, the algorithm outputs a matrix $A$ that satisfies
    $\tfrac{\lambda_{k}(A\Sigma A)}{\lambda_1(A\Sigma A)} \geq \Omega(1)$.
\end{thm}
\begin{proof}
    We prove the theorem by proving the privacy and accuracy
    of Algorithm~\ref{alg:dpp}.
    Privacy follows from Theorems~\ref{thm:dpcp}, \ref{thm:dpfp},
    and~\ref{thm:eigenvalues}, Lemma~\ref{lem:dp-naive},
    and composition of DP (Lemma~\ref{lem:dpcomp}).

    For the accuracy argument, it is enough to show that
    at the beginning of each iteration $1 \leq i \leq k$,
    $$\frac{\lambda_i(A\Sigma A)}{\lambda_1(A\Sigma A)} \geq O(\wb{\gamma}^2).$$
    We prove this via induction on $i$.

    For the basis step, it is trivial because $A\Sigma A = \Sigma$.
    Therefore, the ratio equals $1$.

    Now, we move on to the inductive step. Suppose for $i > 1$, the
    claim holds for all $j < i$. Let the matrix $A$ be equal to
    $A_{i-1}$ at the beginning of iteration $i-1$. This implies that
    for iteration $i-1$,
    \begin{align}
        \frac{\lambda_{i-1}(A_{i-1}\Sigma A_{i-1})}
            {\lambda_1(A_{i-1}\Sigma A_{i-1})} \geq \wb{\gamma}^2.
            \label{eq:inductive-hypothesis}
    \end{align}
    According to the \textbf{If}-block, if the privately estimated
    eigenvalue ratio is less than $4\tau^2$, then it must be the
    case that with high probability (Theorem~\ref{thm:eigenvalues}),
    $\tfrac{\lambda_{i-1}(A_{i-1}\Sigma A_{i-1})}{\lambda_i(A_{i-1}\Sigma A_{i-1})} < 16\tau^2$.
    Then because of \eqref{eq:inductive-hypothesis}, Theorem~\ref{thm:dpcp},
    and Corollary~\ref{coro:naive},
    it must be the case that with probability $1-O(\beta/d)$,
    at the beginning of the nested \textbf{If}-block,
    $$\frac{\lambda_i(BA_{i-1}\Sigma(BA_{i-1})^{\top})}{\lambda_1(BA_{i-1}\Sigma(BA_{i-1})^{\top})}
        \geq \frac{\wb{\gamma}^2}{40}.$$
    Now, if $\tfrac{\lambda_i(Z)}{\lambda_1(Z)} < 4\wb{\gamma}^2$,
    then by Corollary~\ref{coro:naive},
    $$\frac{\lambda_i(BA_{i-1}\Sigma(BA_{i-1})^{\top})}{\lambda_1(BA_{i-1}\Sigma(BA_{i-1})^{\top})}
        < 16\wb{\gamma}^2.$$
    By the guarantees of Theorem~\ref{thm:dpfp}, with probability
    at least $1-O(\beta/d)$, at the end
    of the nested \textbf{If}-block (hence, at the end of the
    loop and the starting of the $i$-th iteration),
    $$\frac{\lambda_i(CBA_{i-1}\Sigma(CBA_{i-1})^{\top})}
        {\lambda_1(CBA_{i-1}\Sigma(CBA_{i-1})^{\top})} \geq \wb{\gamma}^2.$$
    Suppose, the algorithm skips the first \textbf{If}-block.
    Then with high probability, it must be the case that
    $\tfrac{\lambda_{i-1}(A_{i-1}\Sigma A_{i-1})}{\lambda_i(A_{i-1}\Sigma A_{i-1})} \geq \tau^2$.
    If it enters the \textbf{ElIf}-block, then it mean that
    with high probability,
    $$\frac{\lambda_i(A_{i-1}\Sigma A_{i-1})}{\lambda_1(A_{i-1}\Sigma A_{i-1})}
        < 16\wb{\gamma}^2.$$
    Then again, by the guarantees of Theorem~\ref{thm:dpfp},
    with probability at least $1-O(\beta/d)$, at the end of
    the iteration,
    $$\frac{\lambda_i(DA_{i-1}\Sigma(DA_{i-1})^{\top})}
        {\lambda_1(DA_{i-1}\Sigma(DA_{i-1})^{\top})} \geq \wb{\gamma}^2.$$
    This proves the inductive step. If neither of the \textbf{If} or
    \textbf{ElIf}-blocks are entered, it would mean that the ratio is
    already at least $\wb{\gamma}^2$. Applying the union bound over all
    $i$, we get the required result.
\end{proof}

\section{Our Estimator}\label{sec:final}

In this section, we combine the techniques described
thus far, including the DP Preconditioner (Algorithm~\ref{alg:dpp})
and the Naive Estimator (Algorithm~\ref{alg:naive}), and provide
our new estimator for Gaussian covariances, which we call,
"$\DPGCE$". The algorithm first makes the Gaussian well-conditioned
using the preconditioner, followed by estimating it using the naive
estimator, and then it applies the inverse transformation of the
preconditioning matrix. The following is the main result of the
section. Then using that and Lemma~\ref{lem:gaussian-tv}, we would
be able to conclude that
$d_{TV}(\cN(\mu,\Sigma),\cN(\hat{\mu},\wh{\Sigma})) \leq \alpha$.

\begin{thm}
    Let $\Sigma \in \RR^{d \times d}$ be a symmetric, PSD matrix
    and $\mu \in \RR^d$.
    Then for all $\eps,\delta,\alpha,\beta > 0$, there exists an
    $(\eps,\delta)$-DP algorithm that takes
    $$n \geq \wt{O}\left(\frac{d^2}{\alpha^2} + \frac{d^2}{\eps\alpha}
        + \frac{d^{2.5}}{\eps}\right)$$
    samples from $\cN(\mu,\Sigma)$, and outputs a symmetric, PSD
    matrix $\wh{\Sigma} \in \RR^{d \times d}$ and $\hat{\mu}\in\RR^d$,
    such that with probability at least $1-O(\beta)$,
    $$\|\Sigma-\wh{\Sigma}\|_{\Sigma} \leq \alpha ~~~\text{and}~~~
        \|\hat{\mu}-\mu\|_{\Sigma} \leq \alpha.$$
    In the above, $\wt{O}$ hides factors of
    $polylog(d,\tfrac{1}{\eps},\tfrac{1}{\delta},\tfrac{1}{\beta})$.
\end{thm}
\begin{proof}
    In our estimator, Algorithm~\ref{alg:dpgce} is one of the
    main components that is used to estimate the covariance of the Gaussian.
    The other component is the approximate DP version of the private
    mean estimation algorithm ($\PME$) from \cite{KamathLSU19}.
    We replace the preconditioning matrix in $\PME$ by
    our DP preconditioner that we obtain from running
    Algorithm~\ref{alg:dpgce}.
    To prove the theorem, it is enough to show the privacy and
    accuracy guarantees of Algorithm~\ref{alg:dpgce}.
    
    Privacy follows from the privacy guarantees of
    Algorithm~\ref{alg:naive} (Lemma~\ref{lem:dp-naive}),
    Algorithm~\ref{alg:dpp} (Theorem~\ref{thm:dpp}), and
    the approximate DP version of $\PME$ \cite{KamathLSU19},
    followed by composition (Lemma~\ref{lem:dpcomp}) and
    post-processing (Lemma~\ref{lem:postprocessing}).
    
    Now, we prove the first accuracy statement.
    Let $Y$ be the original dataset with $2n$ samples chosen
    i.i.d.~from $\cN(\mu,\Sigma)$. We construct the dataset
    $X$ as follows: for each $i \in [n]$, set
    $X_i = \tfrac{Y_{2i}-Y_{2i-1}}{\sqrt{2}}$. Then each
    $X_i$ is an independent sample from $\cN(0,\Sigma)$.
    We then supply the dataset $X$ to Algorithm~\ref{alg:dpgce}.
    Note that $AX$ contains points from
    $\cN(0,A\Sigma A)$ by construction. This means that
    $\tfrac{\lambda_d(A\Sigma A)}{\lambda_1(A\Sigma A)} \geq \Omega(1)$.
    Thus, by the accuracy guarantees of $\NE$
    (Theorem~\ref{thm:naivepce}), we have
    $\|\Sigma'-A\Sigma A\|_{A\Sigma A} \leq O(\alpha)$.
    However, $\|\Sigma'-A\Sigma A\|_{A\Sigma A} =
    \|\wh{\Sigma}-\Sigma\|_{\Sigma}$. This gives us the first
    result.
    
    The mean estimation result follows from the accuracy
    guarantees of $\PME$, to which we supply the dataset $Y$.
    Note that $\PME$ is designed to
    provide zCDP \cite{BunS16} and has a polylogarithmic dependence
    on the range parameter $R$ that bounds the magnitude of the
    true mean. The goal is to eliminate that dependence, which is
    only possible under approximate DP. The approximate DP version
    of this that doesn't have any dependence on $R$ can be obtained by using the approximate DP version of \cite{KarwaV18} that utilises
    stability based histograms. With a multiplicative cost in
    the sample complexity in terms of $\polylog(1/\delta)$, this
    would establish the result that we need.
\end{proof}

\begin{algorithm}[H]
\label{alg:dpgce}
\caption{Differentially Private
    $\DPGCE_{\eps, \delta, \alpha, \beta}(X)$}
\KwIn{Samples $X_1,\dots,X_{n} \in \R^d$.
    Parameters $\eps, \delta, \alpha, \beta > 0$.}
\KwOut{Matrix $\wh{\Sigma} \in \R^{d \times d}$.}
\vspace{5pt}

\tcp{Precondition the covariance.}
Set $A \gets \DPP_{\eps,\delta,\beta}(X)$.
\vspace{5pt}

\tcp{Estimate the transformed covariance.}
Set $\Sigma' \gets \NE_{\eps,\delta,\beta}(AX)$.
\vspace{5pt}

\tcp{Revert to the original space.}
Set $\wh{\Sigma} \gets A^{-1}\Sigma'A^{-1}$.
\vspace{5pt}

\Return $\wh{\Sigma}$.
\vspace{5pt}
\end{algorithm}

\subsection{Handling the Degenerate Case}

So far, we have implicitly assumed that all the eigenvalues of $\Sigma$
are strictly greater than $0$.
Here, we talk about the case where some of the eigenvalues of $\Sigma$
could be $0$. Let $k \in [d]$ be the largest number such that the
$k$-th eigenvalue of $\Sigma$ is non-zero. Then we can use
Algorithm~\ref{alg:subspace} to exactly recover the top $k$ subspace, and
project onto that subspace, and run $\DPGCE$ within that subspace.
To elaborate, this can be done in three steps: (1) detecting the non-zero
eigenvalues of $\Sigma$ using Algorithm~\ref{alg:eigenvalues};
(2) finding the true subspace of $\Sigma$ using
Algorithm~\ref{alg:subspace}, which can exactly recover the subspace
at a cost of $\wt{O}(d^2/\eps)$ in the sample complexity; and
(3) running Algorithm~\ref{alg:dpgce} on the points projected on
to that subspace.

\printbibliography

\appendix

\section{Naive Estimator}\label{sec:naive}

In this section, we revisit the naive estimator presented in \cite{KamathLSU19} for well-conditioned gaussians. We present a slightly modified version of the algorithm and its analysis that is tailored to our setting.

\bigskip

\begin{algorithm}[H] \label{alg:naive}
\caption{Naive Private Gaussian Covariance Estimation $\NE_{\eps, \delta, \beta}(X)$}
\KwIn{A set of $n$ samples $X_1, \ldots, X_n$ from an unknown Gaussian.
    Parameters $\eps, \delta, \beta > 0$}
\KwOut{A covariance matrix $M$.}
\vspace{10pt}

Set $\hat{\lambda}_1,\dots,\hat{\lambda}_d \gets \DPEE_{\eps,\delta,\beta}(X)$.\\
Set $\kappa \gets 4 \hat{\lambda}_1$.\\
Let $S \gets \left\{ i \in [n]: \| X_i \|_2^2 \leq O(d \kappa_2 \log (n / \beta)) \right\}$\\
Let
$$
\sigma \gets \Theta \left(\frac{d \kappa_2 \log(\frac{n}{\beta})
    \sqrt{\log(1/\delta)}}{n\eps} \right)
$$

Let $M' \gets \frac{1}{n} \sum_{i \in S} X_i X_i^\top + N$ where $N_{ij} \sim \cN(0,\sigma^2)$ \\

Let $M$ be the Euclidean projection of $M'$ on the PSD cone. \\

\Return $M$
\end{algorithm}

\begin{lem}[Analysis of $\NE$]\label{lem:dp-naive}
For every $\eps, \delta, \beta, \kappa_1, \kappa_2, n$, $\NE_{\eps,\delta,\beta}(X)$ satisfies $(\eps,\delta)$-DP, and if $X_1,\dots,X_n$ are sampled i.i.d.\ from $\cN(0,\Sigma)$ for $\kappa_1 \id \preceq \Sigma \preceq \kappa_2 \id$ and satisfy \eqref{eq:cov-cond1}--\eqref{eq:cov-cond3}, then with probability at least $1-O(\beta)$, it outputs $M$ so that:
\begin{enumerate}
    \item $\left\|\Sigma - M\right\|_{\Sigma} \le O \left( \frac{\kappa_2 d^2 \log (n/\beta) \log(1/\beta) \sqrt{\log(1/\delta)}}{\kappa_1n \eps} + \sqrt{\frac{d^2 + \log(1/\beta)}{n}} \right)$.
    \item $\left\|\Sigma - M\right\|_2 \le O \left( \kappa_2 \sqrt{\frac{d + \log(1/\beta)}{n}} + \frac{\kappa_2 d^{3/2} \log (n/\beta) \log(1/\beta) \sqrt{\log(1/\delta)}}{n \eps} \right)$.
\end{enumerate}
\end{lem}

\begin{proof}
We prove the lemma by proving the privacy and accuracy guarantees of Algorithm~\ref{alg:naive}.
We first prove the privacy guarantee.
Given two neighboring data sets $X, X'$ of size $n$ which differ in that one contains $X_i$ and the other contains $X'_i$, the truncated empirical covariance of these two data sets can change in Frobenius norm by at most
\[
    \left\| \frac{1}{n} \left(X_iX_i^\top - X'_i (X'_i)^\top \right) \right\|_F \leq \frac{1}{n} \| X_i \|_2^2 + \frac{1}{n} \| X_i' \|_2^2 \leq O \left( \frac{d \kappa_2 \log (n/\beta) }{n} \right) \; .
\]
\noindent
Thus the privacy guarantee follows immediately from Lemma~\ref{lem:gaussiandp}.

We now prove correctness. First, we have:
\begin{align*}
    \left\|\Sigma - M\right\|_{\Sigma} &\le \left\|M - M'\right\|_{\Sigma} + \left\|M' - \Sigma\right\|_{\Sigma} \\
                                       &\le \|M-M'\|_F\|\Sigma^{-1}\|_2 + \|M'-\Sigma\|_{\Sigma} \\
                                       &\le \sqrt{d}\kappa_1^{-1} \left\|M - M'\right\|_2 + \left\|M' - \Sigma\right\|_{\Sigma} \\
                                       &\overset{(a)}{\le} \sqrt{d}\kappa_1^{-1} \left\|N\right\|_2 + \left\|M' - \Sigma\right\|_{\Sigma} \\
                                       &\le \sqrt{d}\kappa_1^{-1} \left\|N\right\|_2 + \left\|\frac{1}{n} \sum_{i = 1}^n X_i X_i^{\top} - \Sigma\right\|_{\Sigma} + \left\|N\right\|_{\Sigma} \\
                                       &\overset{(b)}{\le} \sqrt{d}\kappa_1^{-1} \left\|N\right\|_2 + \left\|\frac{1}{n} \sum_{i = 1}^n X_i X_i^{\top} - \Sigma\right\|_{\Sigma} + \frac{1}{\kappa_1} \left\|N\right\|_F \\
                                       &\overset{(c)}{\le} O \left( \frac{\kappa_2 d^2 \log (n/\beta) \log(1/\beta)\sqrt{\log(1/\delta)}}{\kappa_1 n \eps} \right) \\
                                       &~~~+ O\left(\sqrt{\frac{d^2 + \log(1/\beta)}{n}}\right) + O \left(\frac{\kappa_2d^{2} \log (n/\beta) \log^{1/2}(1/\beta)\sqrt{\log(1/\delta)}}{n \kappa_1 \eps} \right) \\
                                       &= O \left( \frac{\kappa_2 d^2 \log (n/\beta) \log(1/\beta)\sqrt{\log(1/\delta)}}{\kappa_1 n \eps} + \sqrt{\frac{d^2 + \log(1/\beta)}{n}} \right),
\end{align*}
where $(a)$ holds because $\frac{1}{n} \sum_{i \in S} X_i X_i^\top$ is PSD, and $M$ is the projection of $M' = \frac{1}{n} \sum_{i \in S} X_i X_i^\top + N$ onto the PSD cone, so by Weyl's inequality, the zeroed out eigenvalues have to be at most $\|N\|_2$; $(b)$ is by the inequality $\left\|B^{\frac{1}{2}} A B^{\frac{1}{2}}\right\|_F \le \left\|B\right\|_2 \left\|A\right\|_F$ and the fact that $\Sigma \succeq \kappa_1 \mathbb{I}$; and $(c)$ is due to Facts~\ref{fact:gaussian-facts} and~\ref{fact:chi-squared}.

Additionally, we have:
\begin{align*}
    \left\|\Sigma - M\right\|_2 &\le \left\|\Sigma - M'\right\|_2 + \left\|M' - M\right\|_2 \\
                                &\le \left(\left\|\frac{1}{n} \sum_{i = 1}^n X_i X_i^{\top} - \Sigma\right\|_2 + \left\|N\right\|_2\right) + \left\|N\right\|_2 \\
                                &\overset{(c)}{\le} \left\|\Sigma\right\| \left\|\frac{1}{n} \sum_{i = 1}^n \left(\Sigma^{-\frac{1}{2}} X_i\right) \left(\Sigma^{-\frac{1}{2}} X_i\right)^{\top} - \mathbb{I}\right\|_2 + 2 \left\|N\right\|_2 \\
                                &\overset{(d)}{\le} O \left(\kappa_2 \sqrt{\frac{d + \log(1/\beta)}{n}} \right) + 2 \left\|N\right\|_2 \\
                                &\overset{(e)}{\le} O \left( \kappa_2 \sqrt{\frac{d + \log(1/\beta)}{n}} + \frac{\kappa_2 d^{3/2} \log (n/\beta) \log(1/\beta) \sqrt{\log(1/\delta)}}{n \eps} \right).
\end{align*}
where $(c)$ is by the sub-multiplicative property of the spectral norm, $(d)$ is by Fact~\ref{fact:gaussian-facts} and $(e)$ is by Theorem~\ref{thm:GUE}.
\end{proof}

\begin{cor}\label{coro:naive}
    Suppose $X_1,\dots,X_n$ are sampled i.i.d.\ from $\cN(0,\Sigma)$ for
    $\kappa_1 \id \preceq \Sigma \preceq \kappa_2 \id$ and satisfy
    \eqref{eq:cov-cond1}--\eqref{eq:cov-cond3}. Let $1 \leq k \leq d$ be the largest number, such that
    $\lambda_k(\Sigma) \geq \wb{\gamma}^2\lambda_1(\Sigma)$ for $0 < \wb{\gamma} \leq 1$. If
    $$n \geq O\left(\frac{d^{3/2}\cdot\polylog(1/\beta,1/\delta)}{\eps\wb{\gamma}^2}\right),$$
    then with
    probability at least $1-O(\beta)$, $\NE_{\eps,\delta,\beta,\kappa}(X)$
    outputs $M$ so that for each $1 \leq i \leq k$,
    $\lambda_i(M) \in \left[\tfrac{\lambda_i(\Sigma)}{2},2\lambda_i(\Sigma)\right]$.
\end{cor}
\begin{proof}
    By Lemma~\ref{lem:rotations-cov-conds} and our sample complexity, each
    eigenvalue of $\Sigma$ is estimated correctly by
    the empirical covariance up to a factor of $\sqrt{2}$.
    Now, by Lemma~\ref{thm:GUE} and our sample
    complexity, $\|N\|_2 \in \tilde{O}(\kappa_2\wb{\gamma}^2)$.
    By applying Weyl's inequality (Lemma~\ref{lem:weyl}) for each
    eigenvalue $1 \leq i \leq k$, the claim follows.
    Note, that the eigenvalues corresponding to $i > k$
    may not be estimated accurately, but because $\|N\|_2$
    is bounded, the corresponding estimates in $Z$
    cannot be more than $2\lambda_k(\Sigma)$ by Weyl's
    inequality.
\end{proof}

The following is an immediate consequence of Lemma~\ref{lem:dp-naive}.
\begin{thm} \label{thm:naivepce}
For every $\eps, \delta, \alpha, \beta, > 0, \kappa_2\geq\kappa_1>0$, the algorithm $\NE_{\eps,\delta,\beta}$ is $(\eps,\delta)$-DP, and when given
\[
    n \geq O \left( \frac{d^2 + \log(1/\beta)}{\alpha^2} + \frac{\kappa_2 d^{2} \log (n/\beta) \log(1/\beta)\sqrt{\log(1/\delta)}}{\kappa_1\alpha \eps} \right),
\]
samples from $\cN(0,\Sigma)$ satisfying $\kappa_1\id \preceq \Sigma \preceq \kappa_2\id$, with probability at least $1- O(\beta)$, it returns $M$ such that $\| \Sigma - M \|_{\Sigma} \leq O(\alpha).$
\end{thm}

\end{document}